\documentclass[11pt,a4paper]{article}
\usepackage{authblk}
\usepackage{graphicx}
\usepackage{amsmath}
\usepackage{amssymb}
\usepackage{algorithm}
\usepackage{algorithmic}
\usepackage{amsthm}
\usepackage{epstopdf}
\usepackage{caption}
\usepackage{url}
\usepackage{subcaption}
\usepackage{natbib}
\usepackage{cases}
\newtheorem{theorem}{Theorem}
\newtheorem{lemma}{Lemma}

\theoremstyle{definition}
\newtheorem{assumption}{Assumption}

\title{\bf Fast Global Convergence via Landscape of Empirical Loss    }
\date{}
\author[1]{Chao Qu}
\author[2]{Yan Li}
\author[2]{Huan Xu}
\affil[1]{ Faculty of Electrical Engineering,  Technion}
\affil[2]{H. Milton Stewart School of Industrial and Systems Engineering, Georgia Institute of Technology}
\begin{document}

\maketitle

\begin{abstract} 
	While optimizing convex objective (loss) functions has been a powerhouse for machine learning for at least two decades,    non-convex loss functions have attracted fast growing interests recently, due to many desirable properties such as superior robustness and classification accuracy, compared with their convex counterparts. The main obstacle for non-convex estimators is that it is in general intractable to find the optimal solution.  In this paper, we study the computational issues for some non-convex M-estimators. In particular, we show that the stochastic variance reduction methods converge to the global optimal with linear rate, by exploiting the statistical property of the population loss. En route,  we improve the convergence analysis for the batch gradient method in \cite{mei2016landscape}.

\end{abstract} 

\section{Introduction}
The last several years have witnessed the surge of big data and in particular the rise of {\em non-convex optimization}. Indeed,  non-convex optimization is at the frontier of of machine learning research and an incomplete list includes actively studied problems such as dictionary learning \cite{mairal2009online}, phase retrieval \cite{candes2015phase}, robust regression \cite{huber2011robust} and training   deep neural networks \cite{Goodfellow-et-al-2016}. It is well known that for general non-convex optimization problems, there is no efficient algorithm to find the global optimal solution, unless $P=NP$. Thus, research on  non-convex optimization   can be divided into two categories: The first one drops the requirement of global optimality and seeks a more modest solution concept, such as finding a stationary point. That is,  to show an algorithm  finds a solution $\theta$ such that $ \|\nabla f(\theta)\|^2_2\leq \varepsilon $ \cite{ghadimi2013stochastic,ghadimi2016accelerated,allen2016variance,reddi2016stochastic}. The second one takes  statistical assumptions into consideration, and aims to design algorithms with global convergence under reasonable statistical models \cite{agarwal2010fast,loh2013regularized,qu2017linear,qu2017saga}.  This paper belongs to the second class. Particularly, we consider the following non-convex M-estimator with finite data points:
\begin{equation}
\hat{\theta}=\arg \min_{\theta\in \Omega} R_n(\theta)\equiv \frac{1}{n} \sum_{i=1}^{n} \ell (\theta; (x_i,y_i)),
\end{equation}
where $\ell (\theta,(x_i,y_i))$ is a non-convex loss function for $\theta$, $(x_i,y_i)_{i=1}^n$ are  the sample, $\Omega$ is a convex set, and $\hat{\theta}$ is the global optimal solution. This problem is motivated by the following two examples.

The first example is the binary classification problem. Here,
$$ R_n(\theta)=\frac{1}{n}\sum_{i=1}^{n} \big( y_i-\zeta (\langle \theta, x_i \rangle) \big)^2,$$
where $ y_i\in \{0,1\} $, $ \zeta (\langle \theta, x \rangle)=P(Y_i=1 | X_i=x)$. Several empirical studies have demonstrated superior robustness and classification accuracy of non-convex losses, compared with their convex counterparts \cite{wu2007robust,nguyen2013algorithms}. One popular choice of $\zeta(\cdot)$ is the  logistic loss, which has been used  in   neural networks. 

The second example is   the robust regression problem,  in the following form:
\begin{equation}\label{eq:robust_regression}
R_n (\theta) =\frac{1}{n} \sum_{i=1}^{n} \rho (y_i-\langle \theta, x_i \rangle ).
\end{equation}
The research of robust algorithms for learning and inference was initiated in the 60’s by Tukey \cite{tukey1960survey} and developed rapidly in the 70’s and 80’s \cite{huber2011robust}. One common situation in which robust estimation is used occurs when the data contain outliers, where the ordinary regression method may fail \cite{huber2011robust}. Another situation is that there is a strong suspicion of heteroscedasticity in data, which allows the variance to be dependent on x \cite{tofallis2008least}. This is often the case for many real scenarios.

One renowned loss function in robust statistics is Tukey's bi-square loss,  defined as
$$
\rho_{Tukey}(t)=
\begin{cases}
1-(1-(t/t_0)^2)^3 \quad &\text{for} |t|\leq t_0,\\
1                 \quad & \text{for} |t|\geq t_0.
\end{cases}
$$
It is clear that Tukey's bisquare loss saturates when $t$ is  large  and thus it is non-convex.

While the above formulations have superior statistical properties, a natural question is how to find the global optimal solution of them. In particular if we apply first order methods (the gradient descent method and its variants, including the stochastic gradient method  or stochastic variance reduction methods), what is the theoretical guarantee of the solution? Existing work in literature asserts that the above algorithms converge to the stationary point $\theta$ such that $\|\nabla f(\theta)\|_2^2\leq \varepsilon$ with rate $ \mathcal{O} (\frac{L}{\varepsilon}+\frac{L\sigma^2}{\varepsilon})$ using SGD \cite{ghadimi2013stochastic}, with rate $\mathcal{O}(\frac{nL}{\varepsilon})$ using gradient descent (folklore in optimization community), and with rate $ \mathcal{O} (n+\frac{n^{2/3}}{\varepsilon})$ using SVRG \cite{reddi2016stochastic,allen2016variance}.

This paper aims to provide   stronger results, by a refined analysis making use of the statistical properties of the problem. The high-level intuition is that although the above finite-sum  problem is non-convex, its population counterpart $$R(\theta)=E_{x,y} \ell (\theta; (x,y))  )$$  has good  properties for optimization (although it may still be non-convex). Our work then exploits the resemblance between the population and empirical loss. In a nutshell, when $n$ is large, not only the objective function of the finite sample problem converges to that of the population problem, but both the gradient and the Hessian  converge as well under appropriate regularity conditions, which in literature is the study on the \textit{landscape} of the empirical loss \cite{mei2016landscape}. In particular, \citet{mei2016landscape} proved that the gradient (in $\ell_2$ norm) and the Hessian (in operator norm) converge to its population counterpart with a rate $ \mathcal{O}\sqrt{p\log n/n} $.  Using this tool, they showed a global converge result for  the batch gradient method.
However,  the results in \citet{mei2016landscape} do not make use of the smoothness of the objective function. Instead, they used Lipschitz continuity which leads to a loose rate.  In this paper,  we refine the analysis in the batch gradient method by exploiting the smoothness of the objective function and get a better rate. Moreover, since the  objective functions are of the form of finite-sum of $n$ items and $n$ is large, we apply the Stochastic Variance Reduction Method (SVRG \cite{johnson2013accelerating,xiao2014proximal} and SAGA \cite{defazio2014saga} )  on these problems and established much faster convergence results than the batch gradient method, both  in theory (Section \ref{section:theory})  and in numerical experiments (Section \ref{section:empirical_result}).

We now offer a brief introduction of SVRG and SAGA:
SVRG and related methods have recently surged into prominence for convex optimization given their edge over stochastic gradient descent \cite{johnson2013accelerating,xiao2014proximal,shalev2013stochastic}. Algorithmically, SVRG has inner loops and outer loops. At the beginning of each outer loop, SVRG defines a snap shot vector $\tilde{\theta}$ to be the average (or the last value) of the previous inner loop  and computes the full gradient $\nabla f(\tilde{\theta})$. In the inner loops, it randomly samples a data point $i$ and calculates the variance reduced gradient $$ \nabla f_i(\theta^{k})-\nabla f_i(\tilde{\theta})+\nabla f(\tilde{\theta}).$$ The wisdom of SVRG is by applying this variance reduction technique, the variance of the gradient estimation appraoches to zero, as opposed to   SGD where the variance does not diminish.

Note that SVRG is not a fully ``incremental" algorithm since it needs to calculate the full gradient once in each epoch. SAGA \cite{defazio2014saga}, another popular stochastic variance reduction method, avoids computing the full gradient by storing the historical gradients and then uses them to estimate the full gradient. To achieve this, it pays a price of higher memory demand ($n\times p$ where $p$ is the dimension in general). Nevertheless, in many machine learning problems, the storage demand can be reduced to $\mathcal{O} (n)$, which makes it practical. In these cases, SAGA performs equally or  better than SVRG \cite{defazio2014saga}.  Another merit particularly useful in practice, is that SAGA does not need to tune the length of the inner loop, as opposed to SVRG.  

\textbf{Summary of contribution:} In this paper, we adapt SVRG and SAGA on the non-convex formulation in binary classification and robust regression problems  and prove they converge exactly to the \textit{global optimum}   with \textit{linear} convergence rate by a novel analysis considering the statistical property of the problems. From a high level, we unify the statistics perspective and the optimization perspective for machine learning, with an emphasis on the interplay between them.  These two areas are traditionally studied separately, partly due to the fact that disparate communities developed them. We also improve the analysis of the batch gradient method in \citet{mei2016landscape}. We briefly state the main result and contribution of this paper in the following and leave the details and discussions in Section \ref{section:theory}.
\begin{itemize}
	\item We show the gradient complexity (i.e., the number of gradient evaluation required) of the batch gradient method  is $$\mathcal{O}\big( n (\frac{L}{\mu_0})^2 \log \frac{1}{\varepsilon} \big),$$ where  $L$ is the smoothness of the loss function $\ell(\theta; (x,y))$ and $\mu_0$ is a term similar  to the strong convexity parameter.  The gradient complexities of SVRG and SAGA are $$\mathcal{O}\big( (n +n^{2/3}(\frac{L}{\mu_0})^2) \log \frac{1}{\varepsilon} \big),$$ where $L$ and $\mu_0$ are the same as above. It is clear that, when the condition number $\frac{L}{\mu_0}$ and the number of samples $n$ are large, SVRG and SAGA converge much faster than the batch gradient method. 
	\item Conventional techniques to establish the convergence results for finite sample problems analyze $R_n(\theta)$ directly. The novelty in our analysis is that we firstly analyze $R(\theta)$  to exploit the favorable properties of $R(\theta)$, and then relate that to $R_n(\theta)$. The main challenges in our proofs  for SVRG and SAGA are, besides bounding the deviation between the finite sample problem and the population problem, we need to control the impact of non-convex terms in SVRG and SAGA.
	
	The novelty of our analysis, compared to those in \citet{mei2016landscape} which also makes use of the landscape of empirical loss , is that we exploit smoothness in the analysis.	In particular,  \citet{mei2016landscape} shows the gradient complexity of the batch gradient method is $\mathcal{O}\big( n (\frac{L_{lip}}{\mu_0})^2 \log \frac{1}{\varepsilon} \big)$, where $L_{lip}$ is the Lipschitz continuity parameter. Generally, this Lipschitz continuity parameter can be much larger than the smoothness parameter. Take $f(x)=x^2, \|x\|_2 \leq r $ as a example, the ratio can be as large as the radius $r$. More importantly,  since \citet{mei2016landscape} does not make use of smoothness, their proof technique can not be adapted to the stochastic variance reduction method. 
\end{itemize}


\section*{Related work}

Optimizing non-convex objective functions by the batch gradient and SGD are well studied \cite{nesterov1983method,ghadimi2013stochastic}. The criterion on the convergence is $\|\nabla f(\theta)\|_2^2\leq \varepsilon$ in the smooth and non-constrained problem, and in the constrained or non-smooth regularized case, the Gradient mapping  $ \|\mathcal{G}(\theta)\|_2^2\leq \varepsilon $ is used.  The gradient complexity is shown to be $ \mathcal{O} (\frac{L}{\varepsilon}+\frac{L\sigma^2}{\varepsilon})$ for SGD \cite{ghadimi2013stochastic}, and $\mathcal{O}(\frac{nL}{\varepsilon})$ for gradient descent.

Restricted Strong Convexity (RSC) \cite{negahban2009unified,agarwal2010fast} is a powerful tool to analyze  non-strongly convex and non-convex optimization problems using statistical information in the high dimensional setup. Under RSC, it has been shown that the batch gradient and the stochastic variance reduction methods converge to the global optimum {\em up to the statistical tolerance} \cite{loh2013regularized,qu2017linear,qu2017saga}. The gradient and stochastic variance reduction gradient can reach such tolerance with linear rate. These results  cover problems including Lasso, logistic regression, generalized linear model with non-convex regularization. Compared to these results, our result differs in two ways. First,  our result shows convergence to the exact global optimum rather than its neighborhood.   Second, the conditions of loss functions studied  are different. In their work, the loss function is convex (e.g., the squared loss in Lasso) and thus the whole objective function is either convex or slightly non-convex due to the $ -\frac{\mu}{2} \|\theta\|_2^2$ term in the non-convex  regularization. In our work, each $\ell(\theta; (x_i,y_i))$ is non-convex and RSC does not hold.

There is a vast number of research on stochastic variance reduction methods in the last several years and we   list here a few most relevant ones, see \cite{johnson2013accelerating,defazio2014saga,shalev2013stochastic,schmidt2017minimizing}. In general, if the objective function is $\mu$ strongly convex and the loss function is $L$ smooth, the rate (gradient complexity) is $\mathcal{O} \big( (n+ \frac{L}{\mu})\log \frac{1}{\varepsilon} \big)$  and can be accelerated to $\mathcal{O} \big( (n+ \sqrt{n\frac{L}{\mu}  }) \log\frac{1}{\varepsilon}  \big)$ \cite{lin2015universal,lan2015optimal,zhang2015stochastic}. If the objective function is non-convex, these algorithms converge to a \textit{stationary} point with a sub-linear rate \cite{allen2016variance,reddi2016stochastic}. In  stark contrast,  our paper shows   linear convergence to the global optimum. \citet{shalev2016sdca} proposed the dual-free SDCA algorithm that converges with rate $\mathcal{O} \big((n+\frac{L^2}{\mu^2})\big) \log (1/\varepsilon)$, where each individual loss function can be non-convex but the objective function as a whole is $\mu $ strongly convex.  Recently, several papers have revisit an old idea called P-L (Polyak-{\L}ojasiewicz) condition \cite{polyak1963gradient},  a.k.a. gradient dominated functions, and  proved the linear convergence to the global optimum \cite{karimi2016linear,reddi2016stochastic}. In particular,  \citet{reddi2016proximal}  prove that SVRG and SAGA converges with a rate $\mathcal{O} \big( (n+n^{2/3}\rho) \log(1/\varepsilon)\big) $, where $\rho$ is a condition number depending on P-L condition. Our results are different, since examples we mentioned (binary classification and Tukey's bisquare loss) do not satisfy the PL condition.

\section{Problem setup and Assumption}

\textbf{Binary classification:} In the binary classification problem, we are given $n$ data points $\{ (y_i,x_i)\}_{ i=1,...,n}$, where $x_i\in \mathbb{R}^p$, and $y_i\in \{0,1\}$ is a target with probability $ \mathbb{P}(Y=1| X=x)=\zeta (\langle \theta^*,x\rangle)$,  where $\zeta: \mathbb{R}\rightarrow[0,1]$ is a threshold function, and $\theta^*$ is the true parameter.  We consider the following optimization problem to estimate $\theta^*$:
\begin{equation}\label{equ.binary}
\begin{split}
&\min_{\theta}\quad  R_n(\theta)\triangleq \frac{1}{n}\sum_{i=1}^{n} \big( y_i-\zeta (\langle \theta, x_i \rangle) \big)^2\\
&\text{subject to} \quad \|\theta\|_2\leq r.
\end{split}
\end{equation}

The aim is to find the optimal solution $\hat{\theta}$ of $R_n(\theta)$.
We make the following set of mild assumptions on the threshold function above  and  on the feature vectors, following \cite{mei2016landscape}.
\begin{assumption}\label{Assumption:binary}
	~
	\begin{itemize}
		\item  $\zeta(\alpha)$ is three time differentiable with $\delta'(\alpha)>0$ for all $\alpha$. Furthermore, there exists some constant $L_\zeta>0$ such that  $\max\{\|\zeta'\|_\infty, \|\zeta''\|_\infty,\|\zeta'''\|_\infty  \}\leq L_\zeta$.
		\item The feature vector $X$ is zero mean $\tau^2$ sub-gaussian, i.e., $\mathbb{E} \exp(\langle \lambda, X \rangle )\leq \exp(-\frac{\tau^2\|\lambda^2\|}{2})$.
		\item $\mathbb{E} XX^T\succeq \underline{\gamma}\tau^2 I_{p\times p}$ for some $0<\underline{\gamma}<1$. In other words, the feature vector spans all directions in $\mathbb{R}^p$.
	\end{itemize}	
\end{assumption}

Notice  above assumption on $\zeta$ is quite mild , e.g., it is satisfied by  $\zeta (\alpha)=\frac{1}{1+\exp(-\alpha)}$, which can be used in  neural networks.

\textbf{Robust regression:} We assume the data generation model is $y_i=\langle \theta^*,x_i\rangle+\epsilon_i$. The noise term $\epsilon_i$ are zero mean and i.i.d.. Notice the noise can be very large, e.g., $\epsilon_i$ is sampled from Gaussian mixture distribution $ (1-\delta) \mathcal{N}(0,1)+\delta \mathcal{N}(0,\sigma^2)$ with large $\sigma$, where $\delta$ controls the percentage of large noise. We consider the robust regression in the following form. 
\begin{equation}\label{equ.robust_regression}
\begin{split}
& \min_{\theta} R_n(\theta)\equiv\frac{1}{n} \sum_{i=1}^{n} \rho (y_i-\langle \theta, x_i \rangle ),\\
&\text{subject to} \quad \|\theta\|_2\leq r.
\end{split}
\end{equation}

\begin{assumption}\label{Assumption:robust_regression}
	We define the score function $\psi(z)=\rho'(z)$ and follow the assumption in \cite{mei2016landscape}.
	\begin{itemize}
		\item The score function $\psi(z)$ is twice differentiable and odd in z with $\psi(z)\geq 0$ for all $z\geq 0$. Similar to the binary classification case, we need $ \max \{ \|\psi\|_\infty,\|\psi'\|_\infty,\|\psi''\|_\infty \}\leq L_\psi $.
		\item The feature vector $X$ is zero mean and $\tau^2$ sub-Gaussian random vector, and $\mathbb{E} XX^T\succeq \underline{\gamma}\tau^2 I_{p\times p}$, for some $0<\underline{\gamma}<1$.
		\item The noise $\epsilon$ has a symmetric distribution ($\epsilon$ and $-\epsilon$ have same distribution). Define $g(z)\equiv E_\epsilon (\psi(z+\epsilon))$,  and we have $g(z)>0$ for all $z>0$ and $g'(0)>0$. 
	\end{itemize}
	Remarks: The condition on $g(z)$ is mild .  It is not hard to show it is satisfied if the noise has a density that is strictly positive for all $\epsilon$, and decreasing for $\epsilon>0$.  	
\end{assumption}

\section{Algorithm and Analysis}

The batch gradient descent is the standard one where 
$$ \theta^{k+1}=\theta^k-\eta \nabla R_n(\theta^k)$$
The step size $\eta$ is specified later in the theorem.

We list the algorithm of SVRG and SAGA for completeness. Algorithm \ref{alg:svrg} is  the vanilla SVRG, and  we call Algorithm \ref{alg:svrg} as a subroutine with $J$ times in Algorithm \ref{alg:non-convex-svrg}. Algorithm \ref{alg:non-convex-SAGA} is non-convex SAGA which call algorithm \ref{alg:SAGA} $J$ times as a subroutine, where we follow minibach version of SAGA in \cite{reddi2016proximal}.
\begin{algorithm}[h]
	\caption{SVRG ($\theta^0, T,m,\eta$) }
	\label{alg:svrg}
	\begin{algorithmic}
		\STATE {\bfseries Input:} Total  steps $T$, length of epoch m, $S=\lceil T/m \rceil$, step size $\eta$, initialization $\theta^0$. Set $\tilde{\theta}^0=\theta_m^0=\theta^0$.
		\FOR{$s=0,1,...S-1$}	
		\STATE{ $\theta_0^{s+1}=\theta_m^s$, $\tilde{v}^{s+1}=\frac{1}{n}\sum_{i=1}^{n}\nabla l_i(\tilde{\theta}^s)$  }
		\FOR {$k=0$ \bfseries to $m-1$}
		\STATE {Pick $i_k$ uniformly random from $\{1,...,n\}$}
		\STATE  $v_{k}^{s+1}=\nabla l_{i_k}(\theta_{k}^{s+1})-\nabla l_{i_k}(\tilde{\theta}^s )+\tilde{v}^{s+1}$
		\STATE $\theta_{k+1}^{s+1}= \theta_k^{s+1}-\eta v_k^{s+1}$ 
		\ENDFOR    
		\STATE $\tilde{\theta}^{s+1}=\theta_m^{s+1}$
		\ENDFOR
		\STATE Uniform randomly choose $\theta^j$ from $\{\{\theta_k^{s+1}\}_{k=0}^{m-1} \}_{s=0}^{S-1} $	
	\end{algorithmic}
\end{algorithm}

\begin{algorithm}[h]
	\caption{ non-convex SVRG}
	\label{alg:non-convex-svrg}
	\begin{algorithmic}
		\STATE {\bfseries input:} Total step T, epoch length m, step size $\eta$
		\FOR {$j=0$ \bfseries to $J-1$}
		\STATE $ \theta^j=SVRG(\theta^{j-1}, T,m,\eta) $
		\ENDFOR
		\STATE {output: $\theta^J$}
	\end{algorithmic}
\end{algorithm}

\begin{algorithm}[h]
	\caption{SAGA ($\theta^0, K,\eta,b$) }
	\label{alg:SAGA}
	\begin{algorithmic}
		\STATE {\bfseries Input:} Total  steps $K$, sample size $b$, step size $\eta$, initialization $\theta^0$, $\alpha_i^0=\theta_0$ for $i\in \{1,...,n\}$, $g^0=\frac{1}{n}\sum_{i=1}^{n} \nabla \ell_i(\alpha_i^0)$.
		\FOR{$k=0,...,K-1$}	
		\STATE{ Uniformly randomly pick set $I_k$,$J_k$  from $\{1,...,n\}$ with $|I_k|=|J_k|=b$ }.
		\STATE  $v_{k}=\frac{1}{b} \sum_{i_k\in I_k} (\nabla\ell_{i_k}(\theta^k)-\nabla \ell_{i_k} (\alpha_{i_k}^k))+g^k.$
		\STATE $\theta^{k+1}=\theta^k-\eta v_k$.
		\STATE $\alpha_j^{k+1}=\theta^k$ for $j\in J_k$ and $\alpha_j^{k+1}=\alpha_j^k$ for $j\notin J_k$.
		\STATE $g^{k+1}=g^{k}-\frac{1}{n}\sum_{j_k\in J_k} (\nabla \ell_{j_k}(\alpha_{j_k}^k)-\nabla \ell_{j_k}(\alpha_{j_k}^{k+1})
		)$
		\ENDFOR    
		\STATE Uniform randomly choose $\theta_j$ from $\{\theta^k\}_{k=0}^{K-1}$
	\end{algorithmic}
\end{algorithm}

\begin{algorithm}[h]
	\caption{ non-convex SAGA}
	\label{alg:non-convex-SAGA}
	\begin{algorithmic}
		\STATE {\bfseries input:} Total step K,  step size $\eta$, sample size $b$
		\FOR {$j=0$ \bfseries to $J-1$}
		\STATE $ \theta_j=SAGA(\theta_{j-1}, K,\eta,b) $
		\ENDFOR
		\STATE {output: $\theta_J$}
	\end{algorithmic}
\end{algorithm}

\subsection{Theoretical results}\label{section:theory}

Before we present our main theorems, we brief recall the concentration properties of gradient and Hessian of the empirical loss, since it may give some intuitions why the algorithm works. In our proof, we will use them from time to time. 
\begin{theorem}[Theorem 1 in \cite{mei2016landscape}]\label{Theorem:population}
	Under the assumptions in Binary classification and Robust regression problem, there exists some absolute positive constant $C,C_0, C_1,C_2$, such that following holds ,
	
	\begin{itemize}
		\item The sample gradient converges uniformly to the population gradient in Euclidean norm, i.e., when $n\geq Cp\log n$, we have	
		\begin{align*}
		&\mathbb{P} \left( sup_{\theta\in B^p(r)} \|\nabla R_n (\theta)-\nabla R(\theta)\|_2\leq \tau \sqrt{\frac{Cp \log n}{n}} \right)\\
		&\geq 1-\exp (-C_1 n).
		\end{align*}	
		\item The sample Hessian converges uniformly to the population Hessian in operator norm, i.e., when  $n\geq C_0p\log n$, we have
		\begin{align*}
		&\mathbb{P} \left( sup_{\theta\in B^p(r)} \|\nabla^2 R_n (\theta)-\nabla^2 R(\theta)\|_{op}\leq \tau^2 \sqrt{\frac{C_0p \log n}{n}} \right)\\
		&\geq 1-\exp (-C_2 n).
		\end{align*}
	\end{itemize} 
\end{theorem}
Theorem \ref{Theorem:population} essentially says the gradient and Hessian are close to their population counterparts, if $n\geq p\log n$. For the population loss, It can be shown following the gradient $\nabla R(\theta)$, the gradient decent algorithm converges to the ground truth $\theta^*$, i.e., the optimal solution of $R(\theta)$, even though $R(\theta)$ may be non-convex. Thus, although it is hard to directly analyze the empirical loss, we can exploit the error bound in Theorem \ref{Theorem:population} and wish a similar convergence result. Particularly, since Theorem \ref{Theorem:population} is a uniform converges result, by subtly control the $\|\nabla R_n(\theta_k)-\nabla R(\theta_k)\|_2$ in the algorithm, we can prove the convergence of $\theta^k$.

The next theorem is our main result for binary classification. For ease of exposition in the SVRG and SAGA, in our analysis we assume that $\frac{L^2}{\mu_0^2}\geq n^{1/3}$, a property analogous to the high condition number for strongly convex functions in machine leaning \cite{xiao2014proximal,reddi2016stochastic}.
\begin{theorem}[ Binary classification]\label{Theorem:Binary}
	Let $\hat{\theta}$ be the global optimal solution of Equation \eqref{equ.binary}, and $L$ be the smoothness of loss function $\ell(\theta; (x_i,y_i))$.
	Suppose $\|\theta^*\|\leq \frac{r}{3}$, and Assumption \ref{Assumption:binary} is satisfied, then there exists a positive constant $\mu_0$ that only depends on $L_\zeta, \underline{\gamma}, \tau $, such that	if the sample size $n\geq Cp\log n$, the following   holds with probability at least $1-\exp(-C_1n)$ for some absolute positive constant $C,C_1$:
	\begin{itemize}
		\item Set $\eta=\frac{1}{2L}$ in the batch gradient method,  the algorithm converges to $\hat{\theta}$ with gradient complexity  $\mathcal{O} \big( n(\frac{L}{\mu_0})^2 \log(\frac{1}{\varepsilon})  \big) $.
		\item For SVRG, suppose $\frac{L^2}{\mu_0^2}\geq n^{1/3}$,  we set $T=\lceil C_2 n^{2/3} \frac{L^2}{\mu_0^2} \rceil$ for some absolute positive constant $C_2$, $\eta=\frac{2}{5Ln^{2/3}}$, $m=\lfloor1.25 n \rfloor$, then the algorithm converges to $\hat{\theta}$ with gradient complexity $ \mathcal{O} \big( (n+n^{2/3}\frac{L^2}{\mu_0^2})\log (1/\varepsilon)   \big) $.
		\item For SAGA, suppose $\frac{L^2}{\mu_0^2}\geq n^{1/3},$ we set minibatch size $b=\lceil n^{2/3} \rceil, $ $\eta=\frac{1}{5L}$, $K=\lceil \frac{15L^2}{\mu_0^2} \rceil$, then the algorithm converges to $\hat{\theta}$ with gradient complexity $ \mathcal{O} \big( (n+n^{2/3}\frac{L^2}{\mu_0^2})\log (1/\varepsilon)   \big) $.
	\end{itemize}
\end{theorem}
Some remarks are in order.
\begin{itemize}
	\item We call $\frac{L_0}{\mu_0}$  an effective condition number, an analogy to that in the analysis of strongly convex function. Notice SVRG and SAGA has $n^{1/3}$ less dependence on $\frac{L^2}{\mu_0^2}$ than the batch gradient method.  When the problem is ill-conditioned ($\frac{L}{\mu_0}$ is large ), SVRG and SAGA are much faster than batched gradient method, which is verified in experiments reported in Section \ref{section:empirical_result}. 
	\item  In \citet{mei2016landscape}, the author proves gradient complexity of batch gradient method is $\mathcal{O} (n (\frac{L_{lip}^2}{\mu_0^2})\log(1/\varepsilon) )$, where $ L_{lip}$ is the Lipschitz continuity parameter, since it does not make use of the smoothness of the objective function. Compared to their results, ours is tighter, as $L_{lip}$ is larger than the smoothness parameter $L$ in general.
	\item Most variance reduction work on optimization of strongly convex functions in literature has a linear dependence on $\frac{L}{\mu_0}$, e.g., $n+\frac{L}{\mu_0}$. An open question is whether it is possible to improve our result to $\mathcal{O}\big( (n+n^{2/3} \frac{L}{\mu_0})\log (1/\varepsilon) \big). $ 
	\item Our setting does not satisfied the PL condition \cite{karimi2016linear,reddi2016proximal}, thus existing convergence results based on PL condition does not apply.
	\item  The requirement of $\|\theta\|_2\leq r$ is to ease the theoretically analysis. In practice, we find it not necessary.
\end{itemize}
The next theorem is our main result for robust regression. It is similar to that of the binary classification case except  dependence on $\psi$ and change of constants.
\begin{theorem}[Robust regression]\label{Theorem:Robust regression}
	Let $\hat{\theta}$ be the global optimal solution of Equation \eqref{equ.robust_regression},  and $L$ be the smoothness of loss function $\rho(y_i-\langle\theta, x_i \rangle)$.
	Suppose $\|\theta^*\|\leq \frac{r}{3}$, and Assumption \ref{Assumption:robust_regression} is satisfied. Then there exists a positive constant $\mu_0$ that only depends on $L_\psi, \underline{\gamma}, \tau $,	such that if the sample size $n\geq C_0p\log n$, the following results hold with probability $1-\exp(-C_1n)$, for some absolute positive constant $C_0,C_1$.
	\begin{itemize}
		\item  Set $\eta=\frac{1}{2L}$ in the batch gradient method,   the algorithm converge to $\hat{\theta}$ with gradient complexity  $\mathcal{O} \big( n(\frac{L}{\mu_0})^2 \log(\frac{1}{\varepsilon})  \big) $.
		\item For SVRG, suppose $\frac{L^2}{\mu_0^2}\geq n^{1/3}$,  we set $T=\lceil C_2 n^{2/3} \frac{L^2}{\mu_0^2} \rceil$ with some absolute positive constant $C_2$, $\eta=\frac{2}{5Ln^{2/3}}$, $m=\lfloor2.5 n \rfloor$, then the algorithm converges to $\hat{\theta}$ with gradient complexity $ \mathcal{O} \big( (n+n^{2/3}\frac{L^2}{\mu_0^2})\log (1/\varepsilon)   \big) $.
		\item For SAGA, suppose $\frac{L^2}{\mu_0^2}\geq n^{1/3},$ we set sample size $b=\lceil 2 n^{2/3} \rceil, $ $\eta=\frac{1}{5L}$, $K=\lceil \frac{15L^2}{\mu_0^2} \rceil$, then the algorithm converges to $\hat{\theta}$ with gradient complexity $ \mathcal{O} \big( (n+n^{2/3}\frac{L^2}{\mu_0^2})\log (1/\varepsilon)   \big) $.
\end{itemize} \end{theorem}
The same set of remarks in binary classification holds for robust regression as well. We also remark that 
the basic idea to prove the convergence in robust regression is same with that in binary classification, except that $\mu_0$ has a different dependence on $L_\psi,\underline{\gamma}$ and $\tau$.

\section{Roadmap of the proof}
We briefly explain the high-level idea of the proof while defer the details to the supplementary material.
The proof consists two steps. We take the batch gradient method as an example, since the analysis is relatively easy. The proof for stochastic variance reduction methods is similar, although more involved technically.

We divide the region $B^p(0,r)$, i.e., $\ell_2$ ball in the $p$ dimensional space, into two parts: $$ B^p (0,r)\backslash B^p (\theta^*, \epsilon_0)\quad \text{and} \quad B^p (\theta^*, \epsilon_0). $$
In the first step we focus on $ B^p (0,r)\backslash B^p (\theta^*, \epsilon_0)$. We analyze the objective function $R(\theta)$ (i.e., the population loss) rather than $R_n(\theta)$, such that we can exploit   good statistical properties on the population loss $R(\theta)$ (e.g, the directional gradient toward $\theta^* $ is larger than zero). However, notice   the algorithm only has access to the the gradient of finite sample loss, i.e., $\nabla R_n(\theta)$, rather than $\nabla R(\theta)$. Thus, since the algorithm follows the direction of $\nabla R_n(\theta)$, there are additional error terms on the objective function $R(\theta)$. Thanks to Theorem \ref{Theorem:population},  these terms can be bounded and are small when $n\geq p \log n$. Therefore, we can show in this region, $\theta$ converges toward $\theta^*$ with a linear rate.

The second step, where we analyze the region $\quad B^p (\theta^*, \epsilon_0)$,  is easier. In this region, the population Hessian $\nabla^2 R (\theta) $ can be shown positive definite. Then the empirical Hessian of $\nabla^2 R_n(\theta)$ can be bounded below using the uniform convergence result in Theorem \ref{Theorem:population}. Thus the objective function $R_n(\theta)$ behaves like a strongly convex function in this region, which leads to convergence to $\hat{\theta}$ (notice the uniqueness of $\hat{\theta}$ is proved in Theorem 4 of \citet{mei2016landscape}).

\section{Simulation result}\label{section:empirical_result}
In this section, we report numerical experiment results for SVRG, SAGA, batched gradient and SGD in both synthetic and real datasets.
\subsection{Synthetic dataset}
The aim of the  synthetic data experiment is to verify our main theorem:   SVRG and SAGA converge to the global optimum with linear rate, even in some non-convex settings; they converge faster than the batch gradient method. The step sizes for all algorithms are chosen by a grid search from $\{2^{-10},..., 2^1\}$.

\textbf{Binary classification:} The feature vector $X_i$ is generated from the normal distribution $\mathcal{N}(0,\Sigma).$   The label $Y_i$ is 0 or 1 with probability $P(Y_i=1|X_i=x)=\frac{1}{1+\exp(-\langle \theta^*,x \rangle)}$, where $\theta^*$ is the true parameter. We generate each entry of $\theta^*$ from the Bernoulli distribution with $p=0.5$ and normalize it such that $\|\theta^*\|_2=1.$ In the experiment, the number of data points $n=10000$, the dimension is set as either $p=500$ or $300$. The optimal solution is obtained by running SVRG long enough (e.g., 1000 passes over dataset). We choose two different settings of $\Sigma$, which will affect the condition number of the objective function. In the first experiment $\lambda_{\max}(\Sigma)/\lambda_{\min}(\Sigma)=10 $ and  in the last two experiments $\lambda_{\max}(\Sigma)/\lambda_{\min}(\Sigma)=1000 $. The experiment results are presented in Figure \ref{fig:binary}.

\begin{figure*}
	\includegraphics[width=0.33\textwidth]{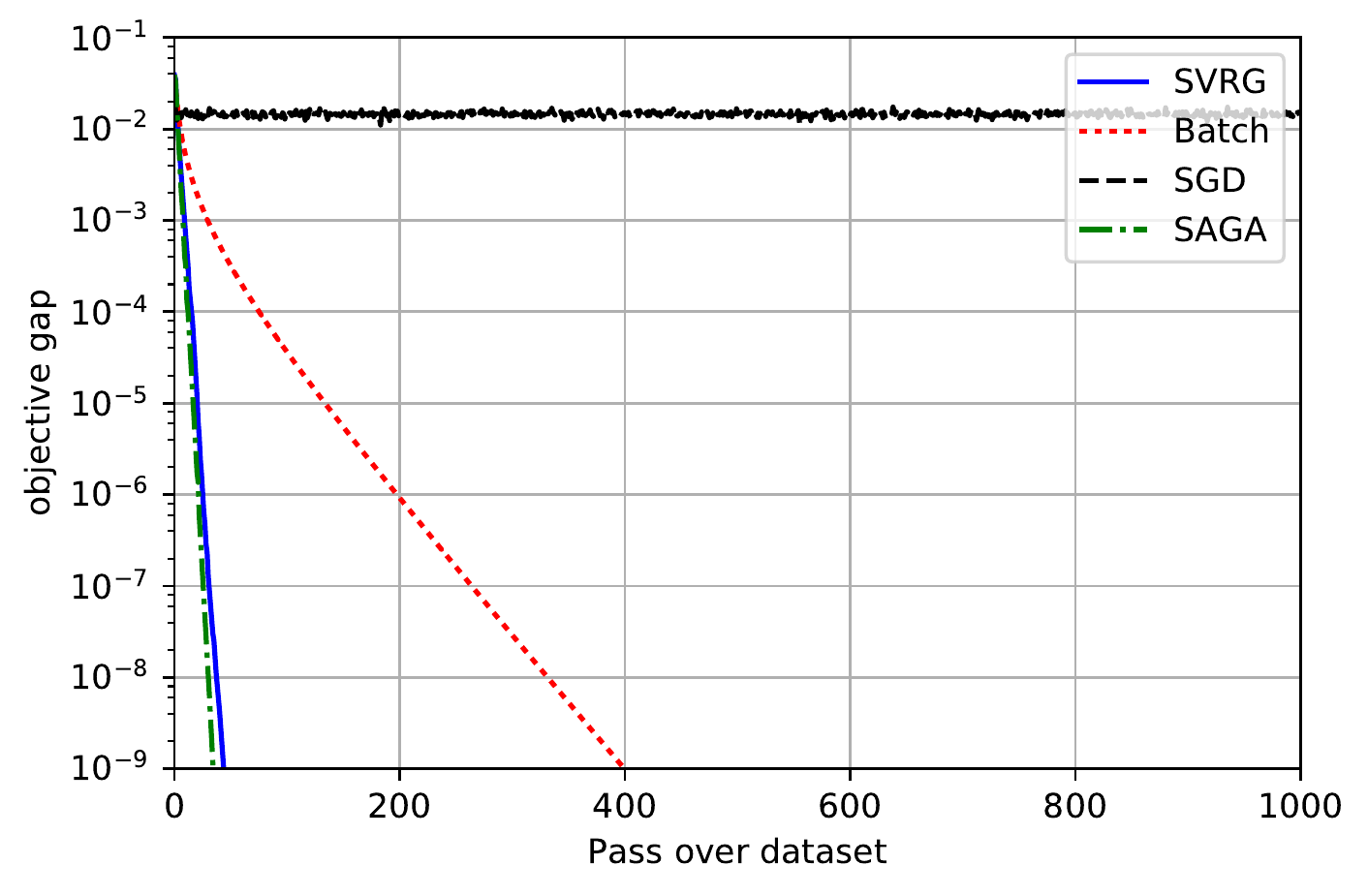}\includegraphics[width=0.33\textwidth]{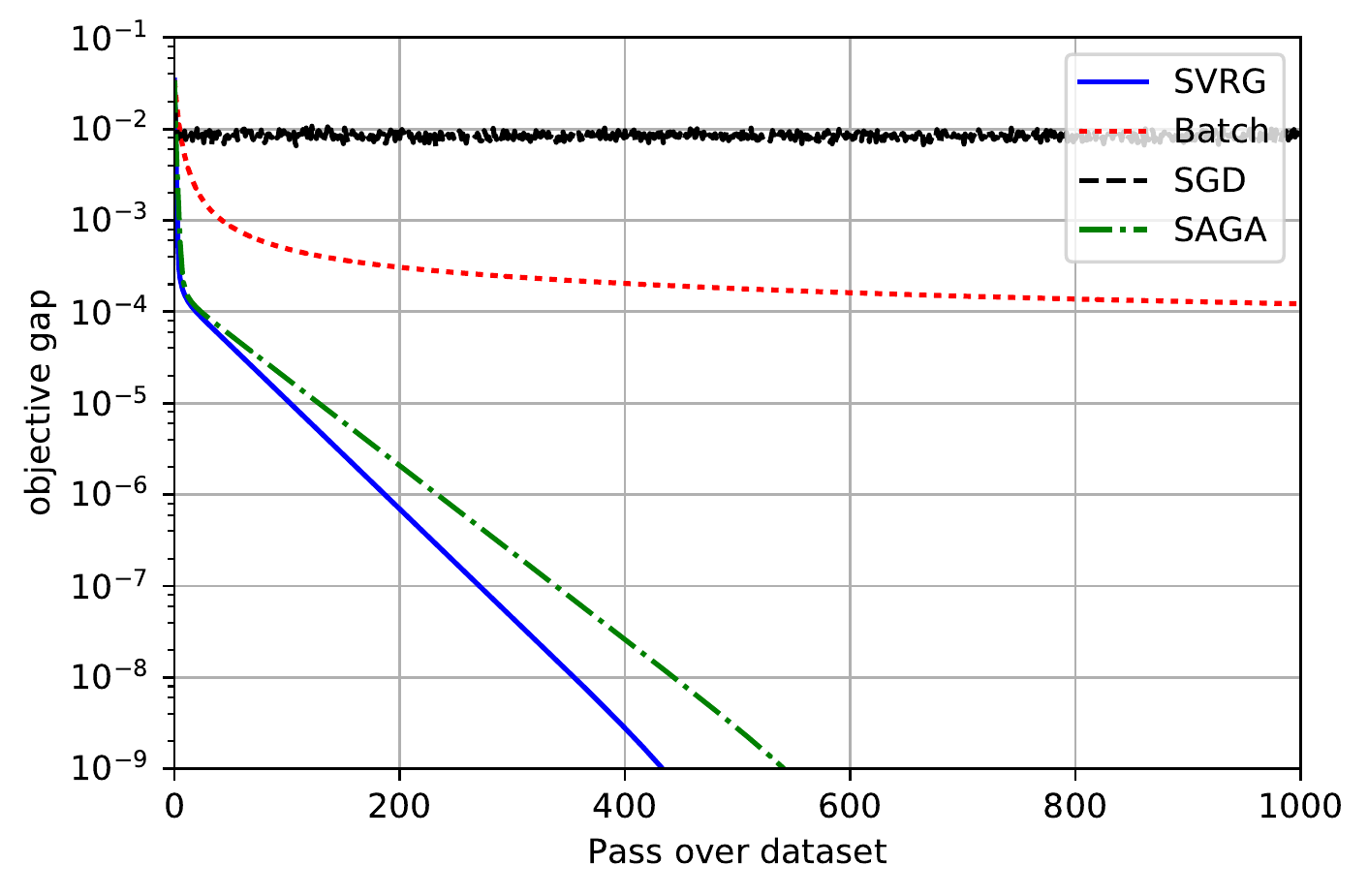}\includegraphics[width=0.33\textwidth]{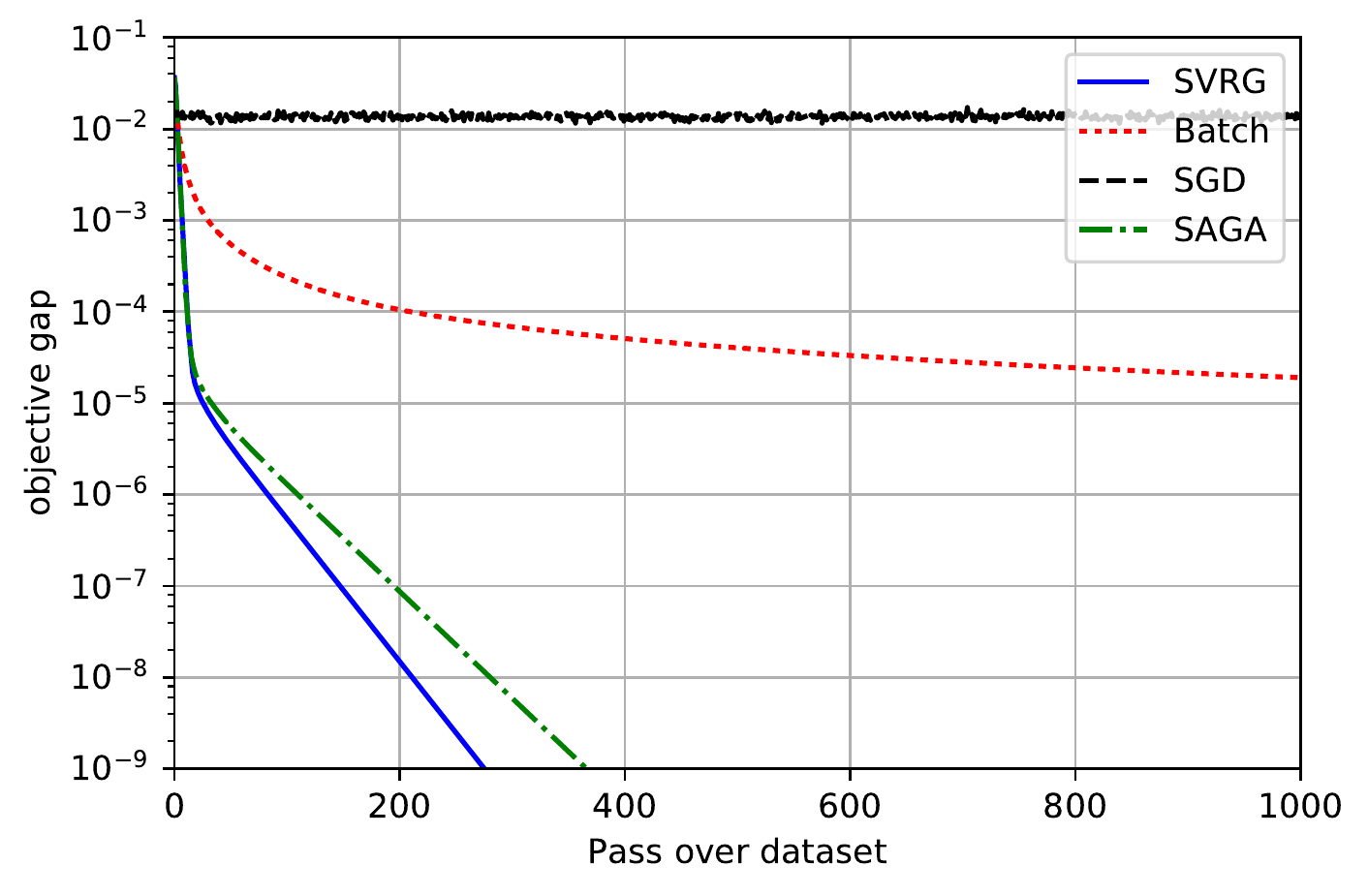}
	\caption{The x-axis is the number of pass over the dataset. y-axis is the objective gap with  log scale. From left to right (a) $n=10000,p=500, \lambda_{\max}(\Sigma)/\lambda_{\min}(\Sigma)=10$.In (b) $n=10000,p=500, \lambda_{\max}(\Sigma)/\lambda_{\min}(\Sigma)=1000$.  (c)  $n=10000,p=300, \lambda_{\max}(\Sigma)/\lambda_{\min}(\Sigma)=1000$. } \label{fig:binary}
\end{figure*}

The experiment results show that SAGA and SVRG have similar performance,  and are followed by batch gradient method  then SGD. When the condition number is small (The left panel in Figure \ref{fig:binary}),  SVRG, SAGA and the batch gradient method all work well, while SVRG and SAGA converge to optimality gap of $10^{-9}$ much faster than the batch gradient method. The middle panel reports the result when the condition number is significantly larger. We observe that even after 1000 passes, the batch gradient method still has a large objective gap ($10^{-4}$), while SVRG and SAGA still work well. In the right panel, we decrease the dimensionality of the problem. Consequently, all algorithms converge faster than the middle panel since there are few parameters. In all settings, SGD converges fast at beginning and then is stuck with a relatively high objective gap ($10^{-2}$) due to the variance of estimating the gradient. 
\begin{figure*}[h]
	\includegraphics[width=0.33\textwidth]{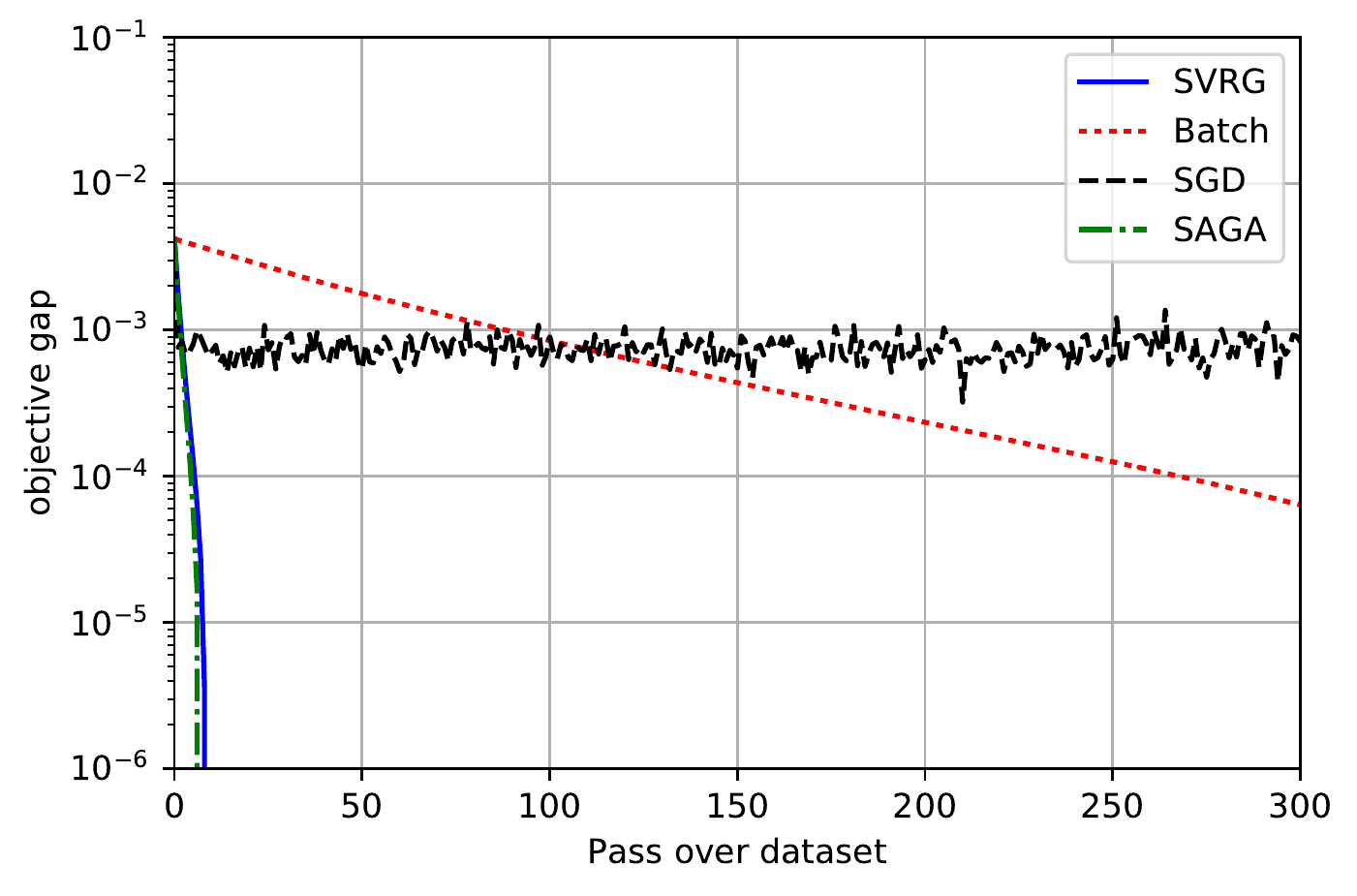}\includegraphics[width=0.33\textwidth]{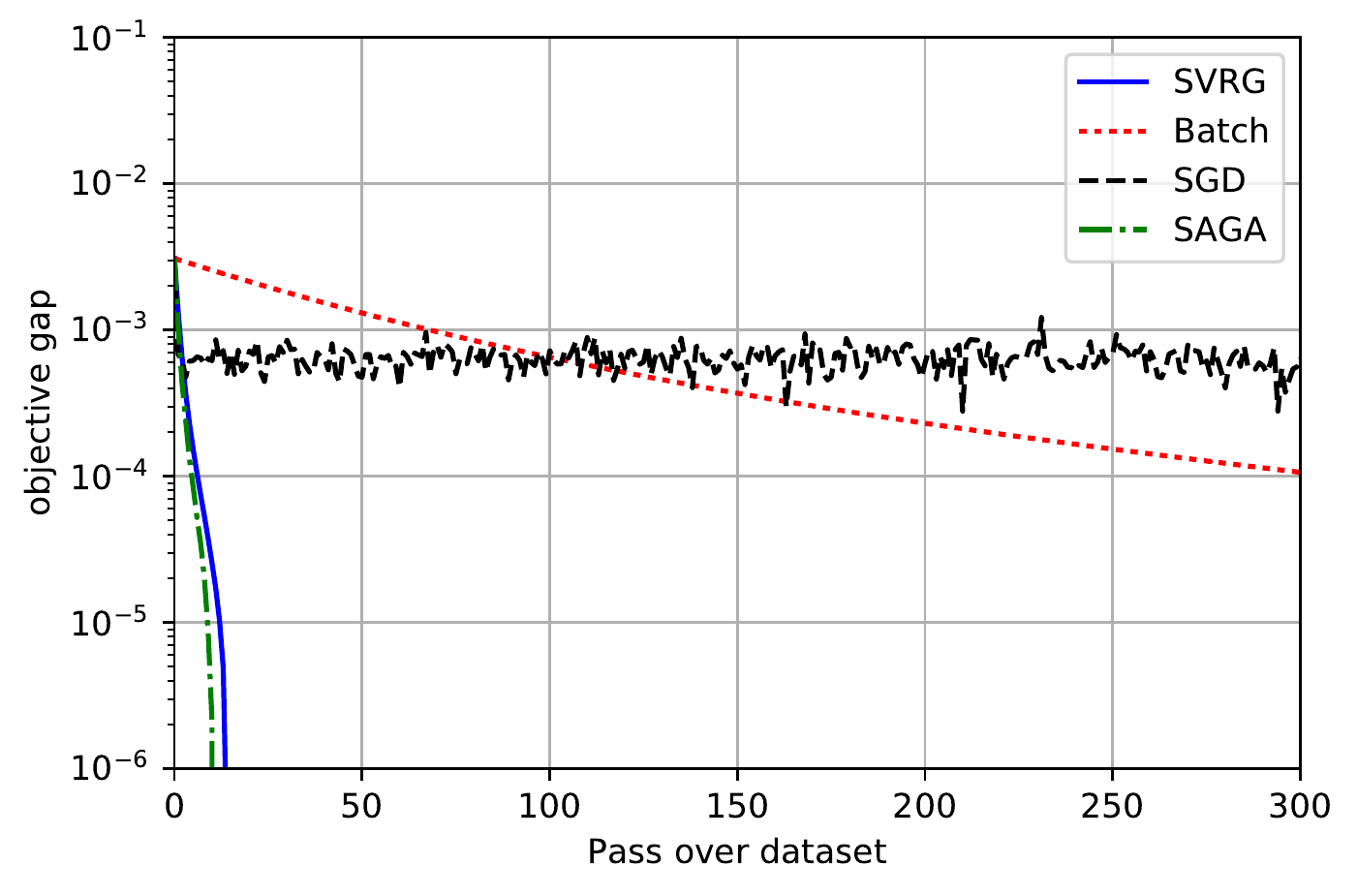}\includegraphics[width=0.33\textwidth]{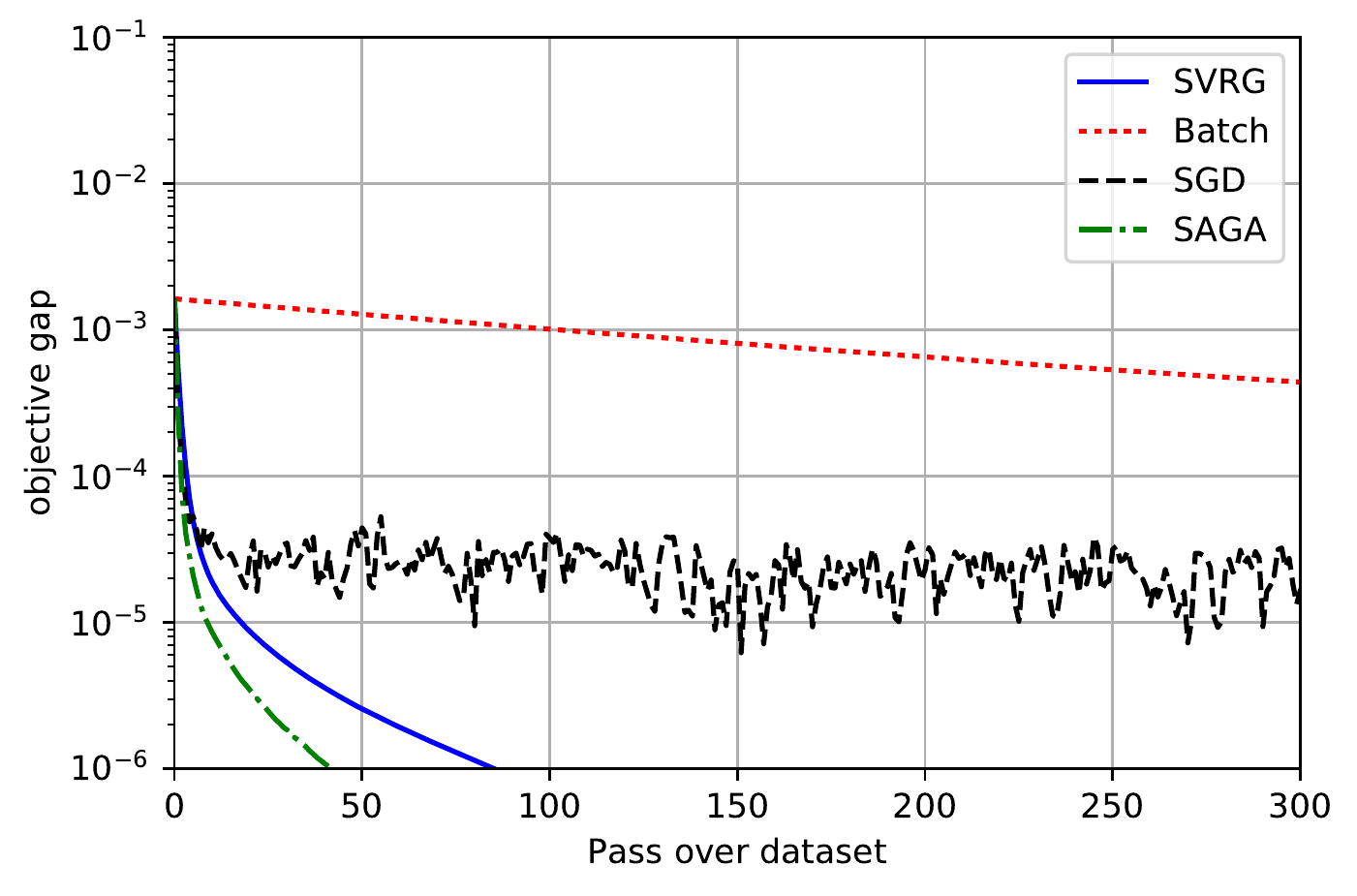}
	
	\caption{The x-axis is the number of pass over the dataset. y-axis is the objective gap with  log scale. From left to right (a) $n=5000,p=500, \lambda_{\max}(\Sigma)/\lambda_{\min}(\Sigma)=100, \delta=0.9, \sigma=10.$  (b)  $n=5000,p=500,\lambda_{\max}(\Sigma)/\lambda_{\min}(\Sigma)=100,\delta=0.8,\sigma=5.$ (c)  $n=5000,p=500,\lambda_{\max}(\Sigma)/\lambda_{\min}(\Sigma)=500,\delta=0.85,\sigma=5.$  } \label{fig:robust}
\end{figure*}

\textbf{Robust regression:} The data generation model  is $Y_i=\langle \theta^*, X_i\rangle+\epsilon_i$, where the feature vector is sampled from the norm distribution $\mathcal{N} (0,\Sigma)$. The noise $\epsilon_i$ is generated from the Gaussian mixture distribution $(1-\delta) \mathcal{N}(0,1)+\delta \mathcal{N}(0,\sigma^2)$.  Each entry of true true parameter $\theta^*$ is generated from the Bernoulli distribution with $p=0.5$. Again, we normalize $\theta^*$ such that $\|\theta^*\|_2=1$. We use Tukey's bisquare loss as our loss function and set $t_0=4.865$ and $r=10$ as that in \citet{mei2016landscape}. The optimal solution is obtained by running SVRG for a long time (1000 passes over dataset) in each experiment. We did three experiments with different settings of $\delta$, $\Sigma$ and $\sigma$, and  present the result in Figure \ref{fig:robust}.

%

In all settings, SAGA performs the best and is followed by SVRG. Both algorithms significantly outperforms the batch gradient method, which verifies our theorem. Even when the problem is  ill conditioned (e.g., $\lambda_{\max}(\Sigma)/\lambda_{\min}(\Sigma)=500$ in the right figure), SAGA and SVRG can converge fast, while batch gradient is very slow. SGD in all settings converges fast at the beginning stage and then is stuck, with the objective gap being $10^{-3}$, due to the variance of  stochastic gradient.

\subsection{Real dataset}
\subsubsection{Binary classification}
In this section, we test the binary classification problem in IJCNN1 (n=49990, p=22) \cite{Danil2001}, Covertype(n=495141, p=54) \cite{blackard1999comparative}, and Dota2 Games result dataset (n=92650,p=116) \cite{Tridgell:2016}. In the experiment, we choose $\zeta(\alpha)=\frac{1}{1+\exp(-\alpha)}$ and normalized all features  to $[-1,1]$. Since Covertype is a multi-class classification dataset, we extract class one and two as our data. We compare SVRG, SAGA, the  batch gradient method and SGD for all three datasets, and present the result in Figure \ref{fig: binary_real}. The optimal solutions for each dataset are obtained by running SVRG long enough time (e.g., 1000 passes of dataset). 

\begin{figure*}
	\includegraphics[width=0.33\textwidth]{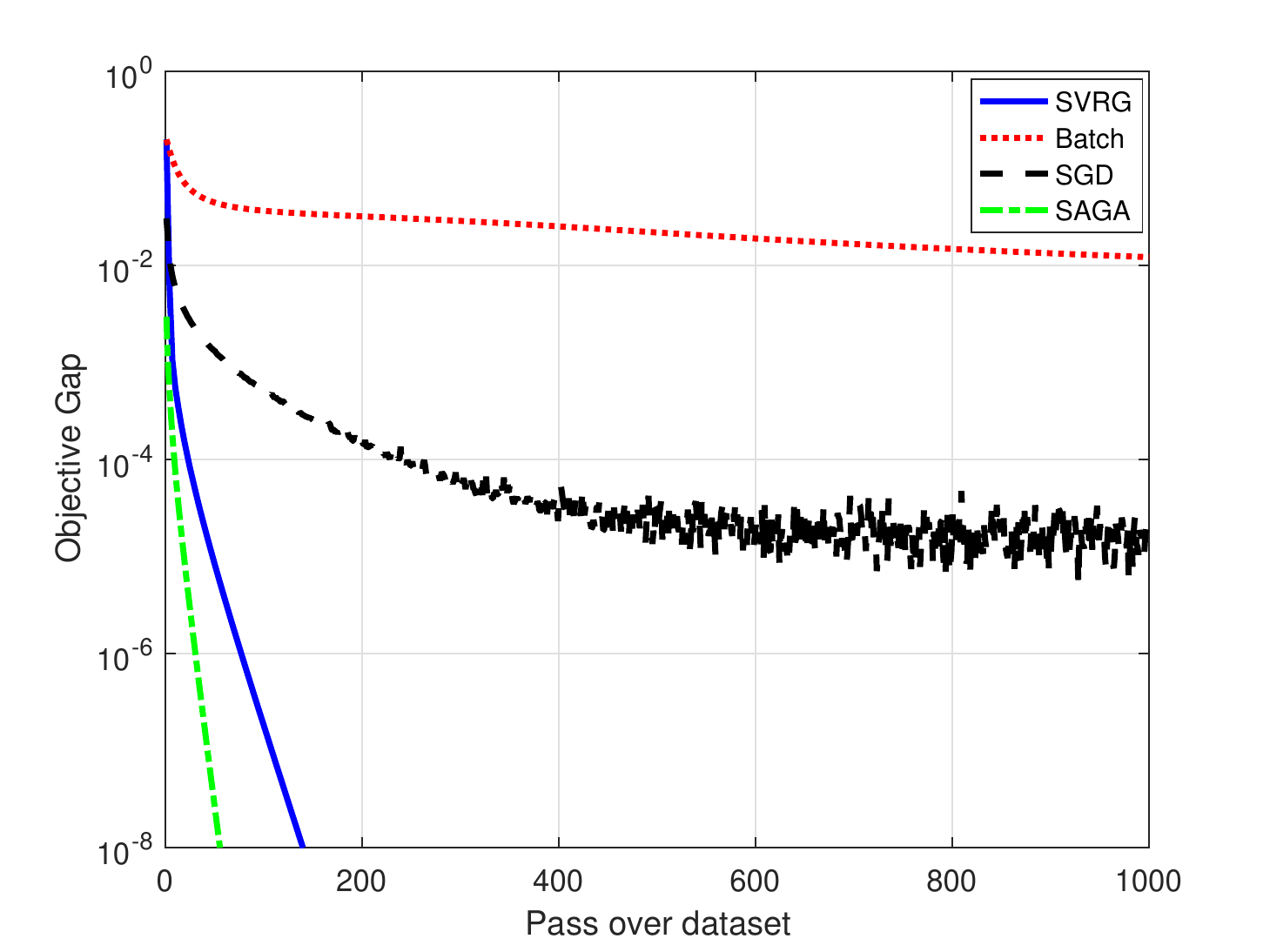}\includegraphics[width=0.33\textwidth]{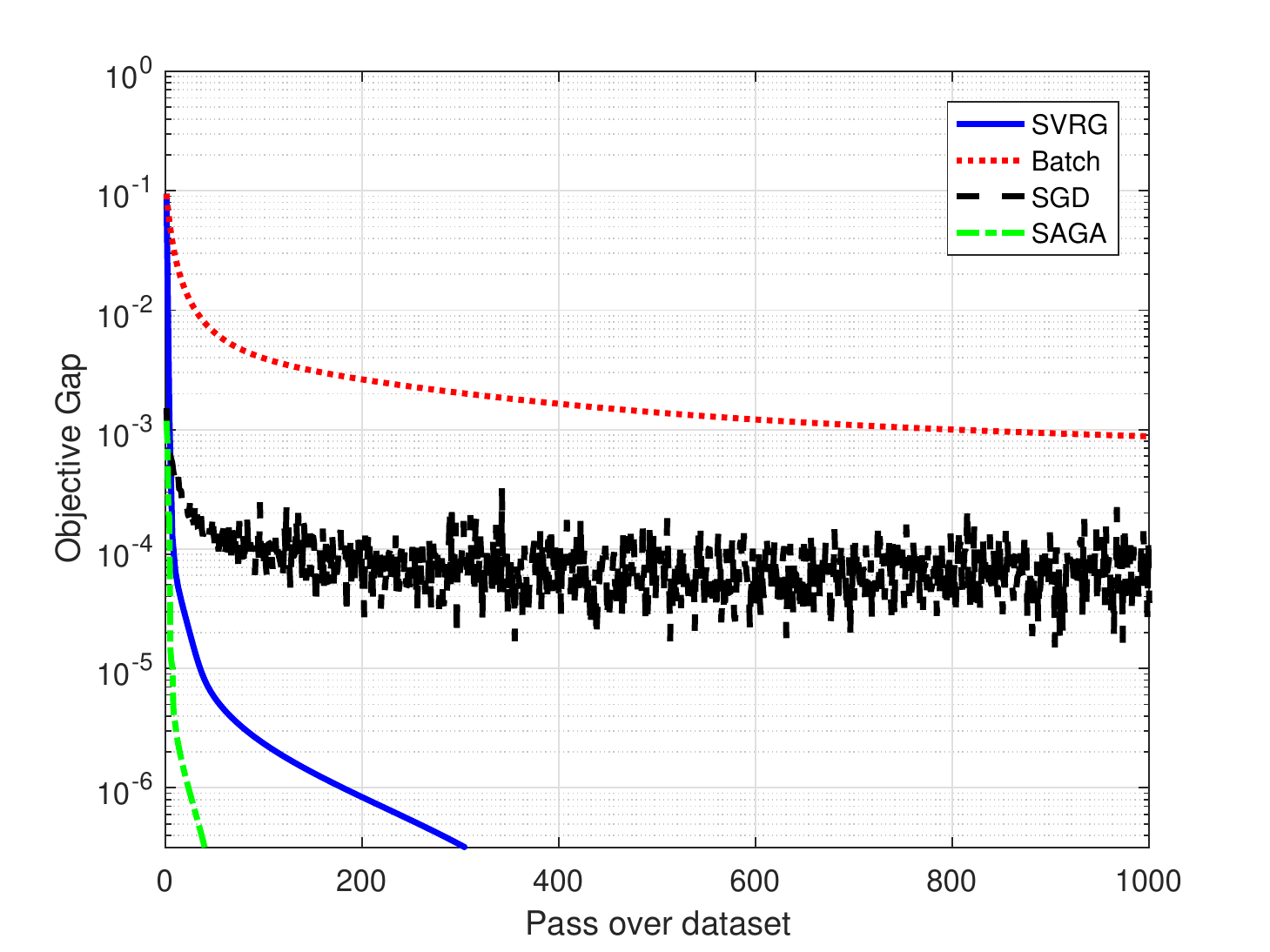}\includegraphics[width=0.33\textwidth]{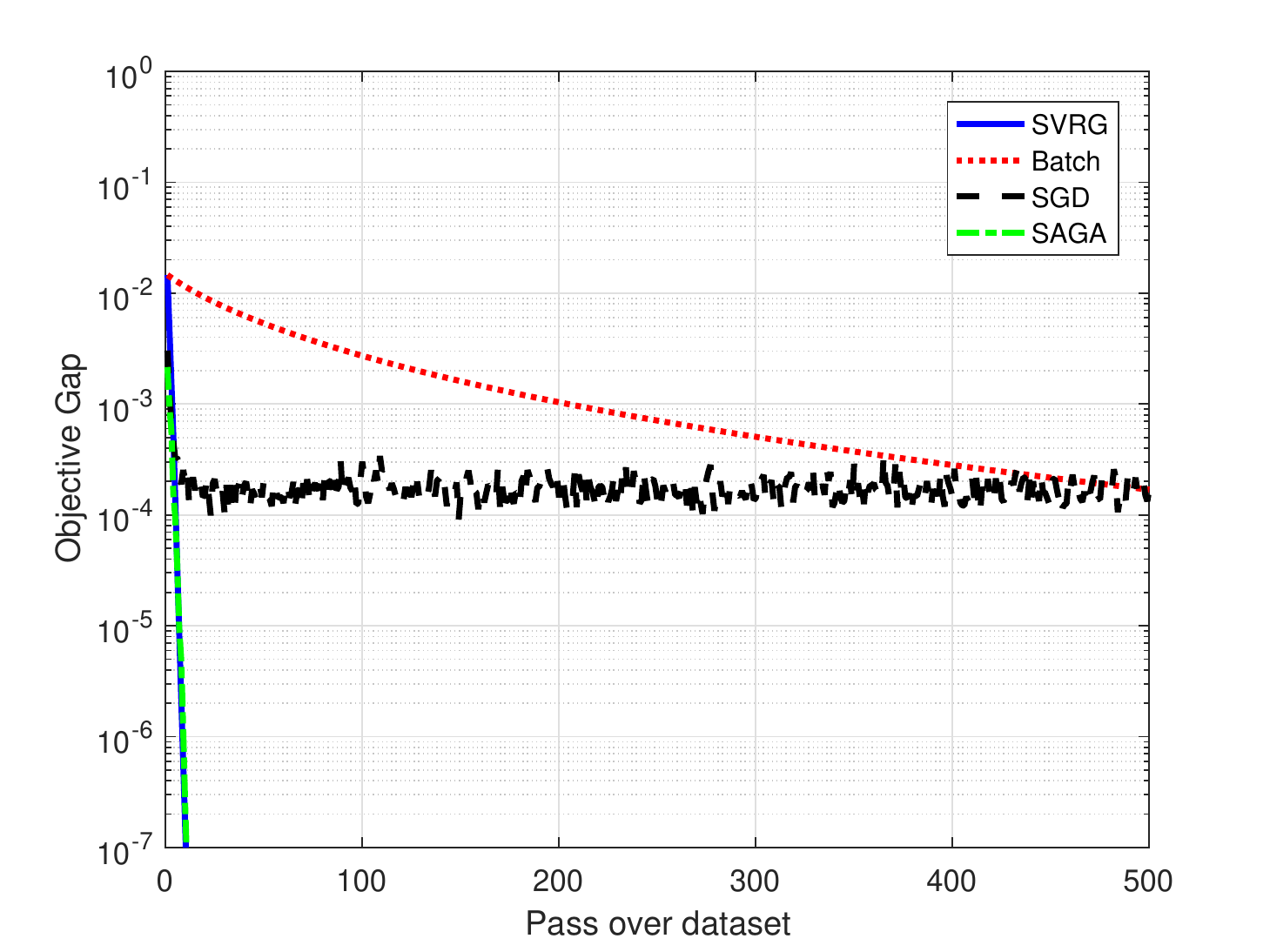}
	\caption{Binary classfication problem in IJCNN1, covertype and Dota2 Game dataset. The x-axis is the passes over dataset, the y-axis is the objective gap with log scale.}	\label{fig: binary_real}
\end{figure*}

SAGA converges fastest in all three experiments, followed  by SVRG. The batch gradient method   converges very slowly due to its bad dependence on the condition number. Indeed it is even worse than the SGD in all three datasets. SGD converges very fast at early stage and  saturates at a large objective gap.
\begin{figure*}
	\includegraphics[width=0.33\textwidth]{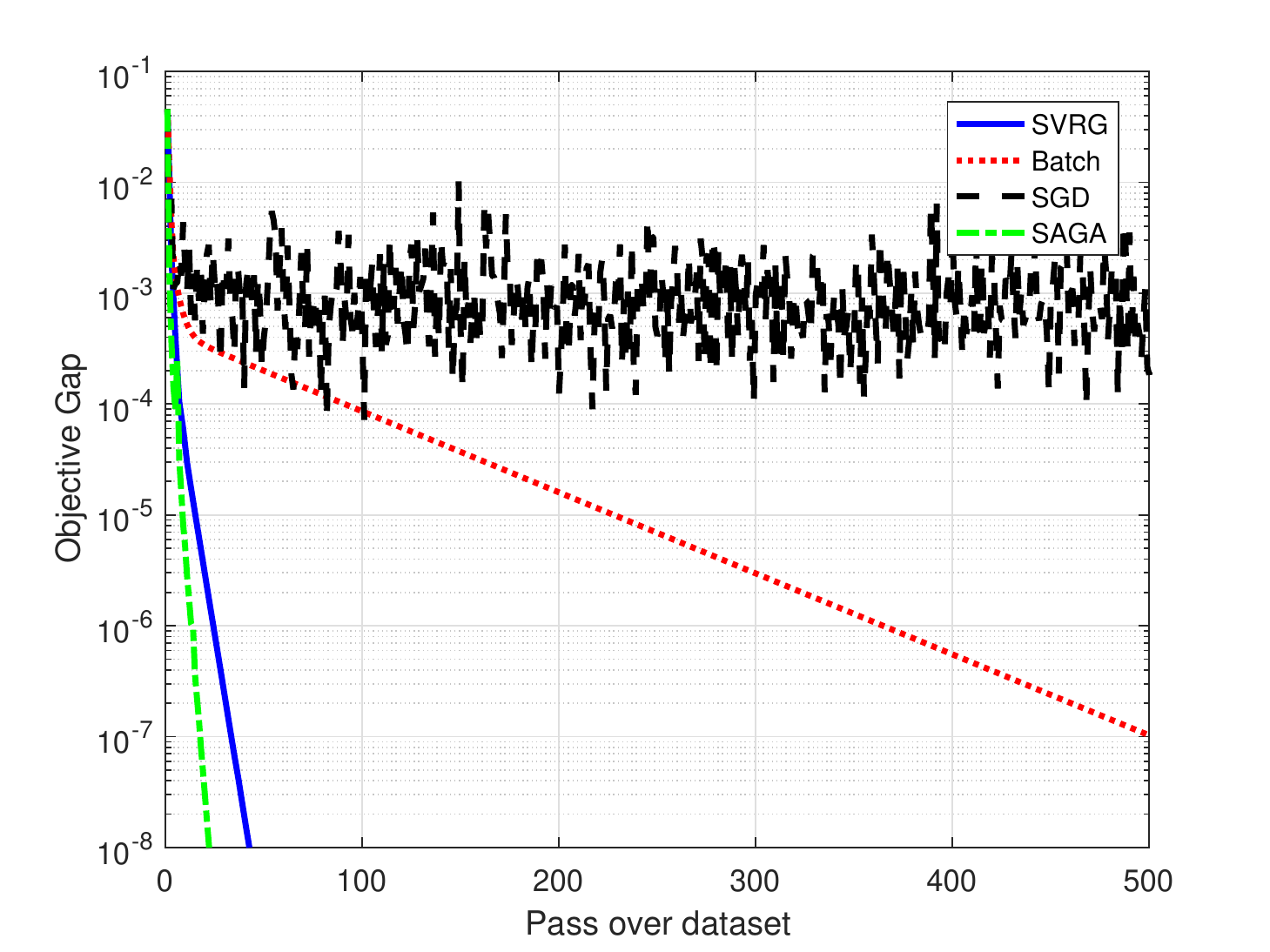}\includegraphics[width=0.33\textwidth]{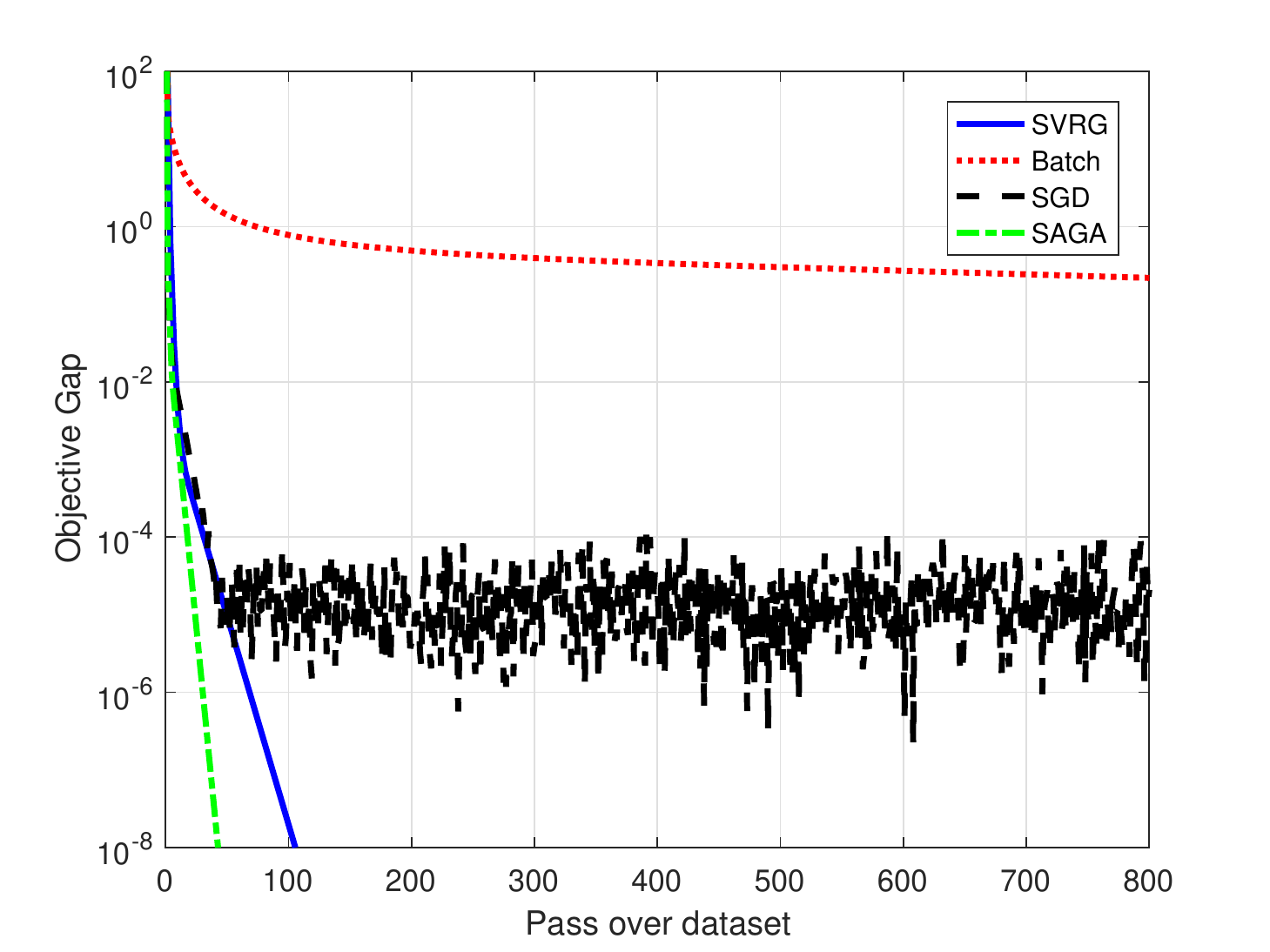}\includegraphics[width=0.33\textwidth]{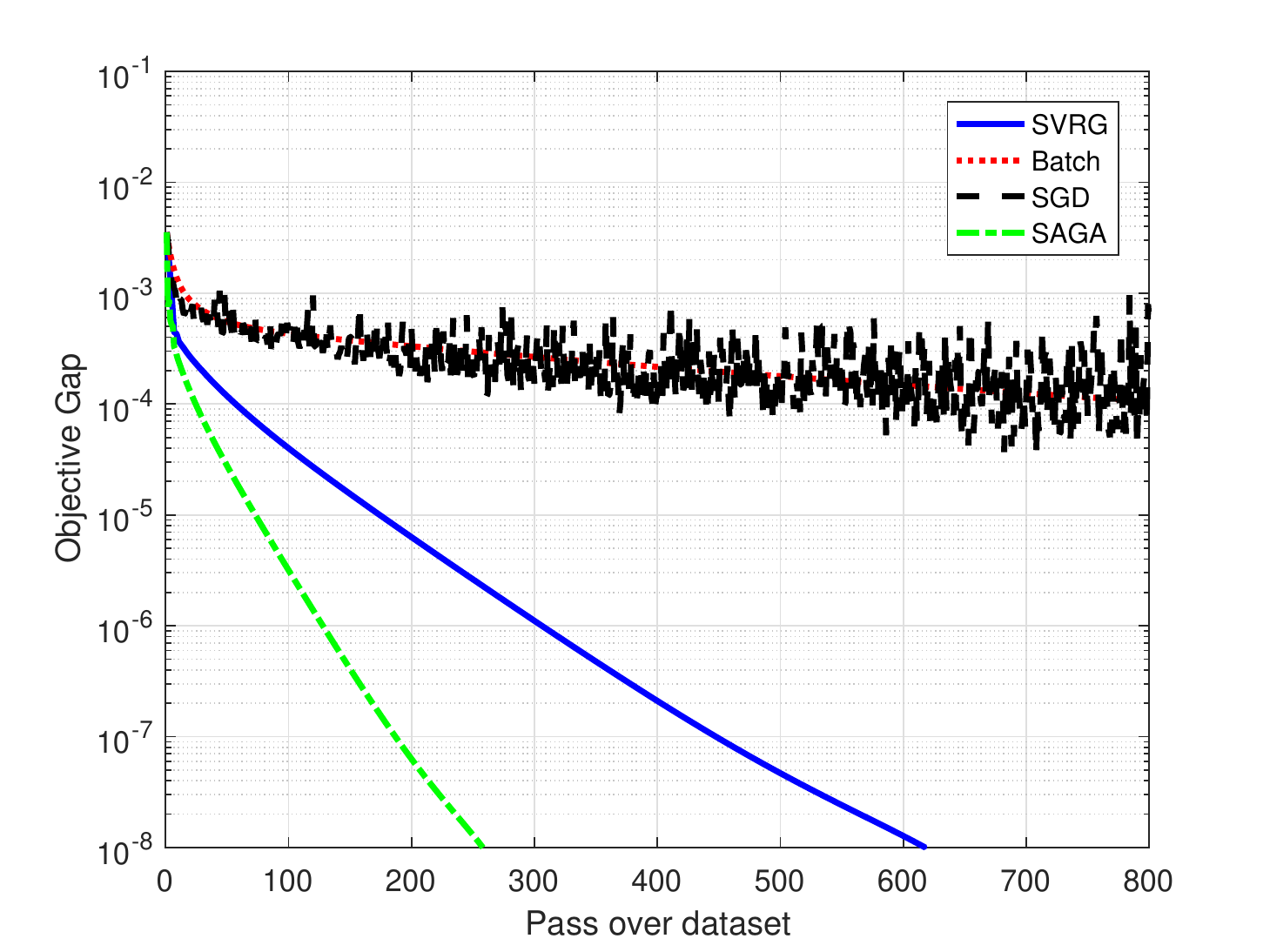}
	\caption{Robust regression problem in Airfoil Self-Noise, Communities and Crime, and Parkinsons Telemonitoring dataset. The x-axis is the pass over dataset, the y-axis is the objective gap with log scale.}	\label{fig: robust_real}
\end{figure*}

%

\subsubsection{Robust regression}

We test the robust regression problem for the following datasets: Airfoil Self-Noise (n=1503, p=6) \cite{Brooks2014}, Communities and Crime (n=1994,p=128) \cite{redmond2002data}, and Parkinsons Telemonitoring  (n=5875,p=26) \cite{tsanas2010accurate}. We corrupt the output by adding   noise from the Gaussian mixture distribution $\epsilon \sim (1-\delta) \mathcal{N}(0,1)+\delta \mathcal{N}(0,\sigma^2) $, similarly as in the synthetic dataset, and choose $\delta=0.1$ in all three experiment, $\sigma=2,5,5$ respectively. The result of  SVRG, SAGA, batched gradient method and SGD   on these three datasets are reported in Figure \ref{fig: robust_real}.

In all three experiments, SAGA converges the fastest,  followed by SVRG. SGD in all experiments converges quickly at the beginning stage and then is stuck at large objective gaps later due to variance of the stochastic gradient.  In the dataset Airfoil Self-Noise, batched gradient performs well. However at the other two datasets, its performance is either similar to SGD (Prikinsons Telemonitoring dataset), or even worse  (Communities and Crime). This is likely due to its heavy dependence to the condition number.

\section{Conclusion and future work}

In this paper, we solve two kinds of non-convex problem with stochastic variance reduction methods and prove the algorithms converge to the global optimum of the problem linearly. Our analysis exploits the fact that the population problem often has more favorable properties in terms of optimization.  Although the finite sample problem does not possess these favorable properties, it is possible to control the impact of departing from the population problem on the performance of optimization algorithms.

A future work is to consider the optimization in the high dimensional statistics setting, i.e., when $p\gg n$. In this case, Theorem \ref{Theorem:population} fails and a possible solution is to add the regularization to encourage the sparsity. 

\bibliography{landscape}
\bibliographystyle{plain}

\onecolumn
\appendix
\title{\bf Appendix }
\date{}
\author{}
\maketitle

\section{Binary classification}	

We start with a Lemma which presents some properties of $R(\theta)$. It is quite similar to Lemma 8 in \cite{mei2016landscape}, and we present here for completeness. 
\begin{lemma}\label{Lemma.bound_gradient_binary}
	Assume $\|\theta^*\|\leq r/3$, and Assumption \ref{Assumption:binary} is satisfied, 
	\begin{itemize}
		\item There exist an $\epsilon_0$ and $\kappa_0>0$ such that $\inf_{\theta\in B^p(\theta^*,\epsilon_0)} \lambda_{\min} \big(\nabla^2 R(\theta) \big)\geq \kappa_0 $.
		\item There exist some positive constant $\underline{L}_0$ such that $ \inf_{\theta \in B^p (0,r)\backslash B^p(\theta^*,\epsilon_0)} \|\nabla R(\theta)\|_2\geq \underline{L}_0$.
		\item  For all $\theta\in B^p (0,r)\backslash B^p (\theta^*, \epsilon_0/2)$, $\|\nabla R_n(\theta)\|_2>0$ 	
	\end{itemize}
	All constants $\epsilon_0, \mu_0, \underline{L}_0$ just depend on $\sigma(\cdot),r,\tau,L_\sigma$. 
\end{lemma}

\begin{proof}
	
	It is easy to verify that $$\nabla^2 R(\theta)=\mathbb{E} \{\beta(\theta) XX^T\}, $$
	where $\beta(\theta)=2\big( \zeta'(\theta^T X)^2 + (\zeta (\theta^TX)-\zeta(\theta^{*T}X)) \zeta''(\theta^TX) \big)$, and we use the fact that $\mathbb{E} Y=\zeta (\langle 
	\theta^*,X\rangle).$
	
	At the ground truth $\theta^*$, we have	
	$$\nabla^2 R(\theta^*)=\mathbb{E} \big(2\zeta'(\theta^{*T}X)^2XX^T \big) .$$
	
	Recall in Assumption \ref{Assumption:binary}, we have $\zeta'(\cdot)>0$, thus for any $s>0$ there exist $L(s)$ such that
	\begin{equation}\label{equ.L_s}
	\zeta'(t)\geq L(s), \forall t\in [-s,s].
	\end{equation}
	We define a event $B=\{|\langle \theta^*,X \rangle|\leq \frac{2s}{3} \}$, then for any $\|u\|_2=1$, we have
	
	\begin{equation}
	\begin{split}
	u^T \nabla^2 R(\theta^*) u &\geq \mathbb{E} \{2\zeta'(\theta^{*T} X)^2 \langle u,X \rangle^2 \mathbf{1}_{B}  \}\\
	& \geq 2L^2(s) \mathbb{E} \{ \langle u,X\rangle^2 \mathbf{1}_{B} \}\\
	&\overset{(a)}{\geq} 2L^2(s) \big( \mathbb{E} [\langle u,X \rangle^2 ]-\mathbb{E} [\langle u,X \rangle^2] \mathbf{1}_{B^c} \big)\\
	&\overset{(b)}{\geq} 2L^2(s) \big( \underline{\gamma}\tau^2-(\mathbb{E} \langle u,X \rangle^4  \mathbb{P}(B^c) )^{1/2}   \big)\\
	&\overset{(c)}{\geq} 2L^2(s) \big( \underline{\gamma}\tau^2-\tau^2\sqrt{ C_4 \mathbb{P}(B^c) } \big)\\
	& \overset{(d)}{\geq} 2L^2(s) \tau^2 (\underline{\gamma}-\sqrt{C_4}\exp(-\frac{s^2}{r^2\tau^2})),
	\end{split}
	\end{equation}
	where (a) uses equation \eqref{equ.L_s} , (b) holds from the Cauchy-Schwartz inequality and  (c) holds from the property of sub-Gaussian variable as follows \cite{vershynin2010introduction}:  Suppose $X$ is zero mean and $\tau^2$ sub-Gaussian random variable, there exists numerical contants $C_{2k}\in (0,\infty)$ for all integers $k\geq 1$ such that 
	
	\begin{equation}\label{equ.subgaussian1}
	\mathbb{E} |\langle u,X\rangle|^{2k} \leq C_{2k}\|u\|_2^{2k} \tau ^{2k}.
	\end{equation}

	(d)  uses the concentration of sub-Gaussian variable, we refer reader to \cite{vershynin2010introduction}.

	Thus we can choose $s=Cr\tau$ with some positive constant $C$ such that $ \sqrt{C_4}\exp(-C)\leq \frac{1}{2}\underline{\gamma}  $ and then have 
	$$ u^T\nabla^2 R(\theta^*) u\geq L^2(s) \tau^2\underline{\gamma} .$$
	
	Now we bound $\nabla^2 R(\theta)$ by bound the difference of $\nabla^2 R(\theta)-\nabla^2 R(\theta^*).$ We denote the Lipschitz parameter of $\beta(\theta)$ as $L_\beta$ (notice it just depends on $L_\sigma$).
	For $\|u\|_2=1$, we have 
	\begin{equation}
	\begin{split}
	|u^T (\nabla^2 R(\theta)-\nabla^2 R(\theta^*))u|=&\mathbb{E} \{ u^T [(\beta(\theta)-\beta(\theta^*))XX^T]u \}\\
	&\leq L_\beta \mathbb{E} \{\langle \theta-\theta^*,X \rangle\cdot \langle u,X \rangle^2   \}\\
	&\leq L_\beta [\mathbb{E}\langle\theta-\theta^*,X \rangle^2\cdot \mathbb{E}\langle u,X \rangle^4]^{1/2}\\
	&\leq L_\beta \big(  C_4\|\theta-\theta^*\|_2^2\tau^6 \big)^{1/2}\\
	&=L_\beta \sqrt{C_4} \|\theta-\theta^*\|_2\tau^3,
	\end{split}
	\end{equation}
	where the first inequality uses the fact that $\beta$ is $L_\beta$ Lipschitz, second ones holds from the Cauchy-Schwartz inequality, the third one uses equation \eqref{equ.subgaussian1} again.
	
	Now we set $\epsilon_0=\frac{L^2(s) \underline{\gamma}}{2\sqrt{C_4}L_\beta \tau}$.
	
	It is clear that when $\|\theta-\theta^*\|\leq \epsilon_0$, 
	
	\begin{equation}\label{equ.smallest_eig}
	\lambda_{\min} (\nabla^2 R(\theta^*) )\geq \kappa_0=\frac{1}{2} L^2(s)\tau \underline{\gamma},
	\end{equation}
	where $s=C\gamma\tau$.

	Then we lower bound the magnitude of the gradient $\nabla R(\theta)$ when $\theta \in B^p (0,r)\backslash B^p(\theta^*,\epsilon_0)$. 
	
	Recall we have $R(\theta)=\mathbb{E} (Y-\zeta (\langle \theta, X \rangle))^2$. It is easy to verify that $R(\theta)$ is minimized at  the truth parameter $\theta^*$.
	Notice 
	\begin{equation}
	\begin{split}
	\langle \theta-\theta^*, \nabla R(\theta) \rangle&=\langle \theta-\theta_0, X \rangle \cdot 2\mathbb{E} \{(\zeta (\theta^TX)-Y ) \zeta' (\theta^T X)\}\\
	&=2\mathbb{E} \{ (\zeta (\theta^TX)-\zeta(\theta^{*T} X)) \zeta' (\theta^TX)\cdot \langle \theta-\theta^*,X \rangle  \}
	\end{split}
	\end{equation} 
	where we use the fact that $\mathbb{E} Y=\zeta (\langle 
	\theta^*,X\rangle)$.

	We define some event $A$ such that  $ | \theta^TX|\leq s, |\theta^* X|\leq s, |(\theta-\theta_0)^TX)|\leq s $ when this event happens. Particularly, let $U\in \mathbb{R}^{2\times d}$ be an orthogonal transform from $\mathbb{R}^p$ to $\mathbb{R}^2$, whose row space contain $\theta$ and $\theta^*$ and $A=\{ \|UX\|_2\leq \frac{2s}{3r}  \}$. Since $\theta\in B^p (0,r), \|\theta^*\|\leq \frac{r}{3}$, thus we have $ | \theta^TX|\leq s, |\theta^* X|\leq s, |(\theta-\theta^*)^TX)|\leq s$ when $A$ happens.
	Then we have
	\begin{equation}
	\begin{split}
	\langle \theta-\theta^*,\nabla R(\theta) \rangle &\geq 2L^2(s) \{ \mathbb{E} \langle \theta-\theta^* ,X\rangle^2 -\mathbb{E} \langle \theta-\theta^* ,X\rangle^2 \textbf{1}_{A^c}  \}\\
	& \geq 2L^2(s) \{\underline{\gamma} \tau^2 \|\theta-\theta^*\|^2- \big( \mathbb{E} [ \langle \theta-\theta^*,X \rangle^4] \mathbb{P}(A^c)     \big)^{1/2}  \}\\
	& \geq 2L^2(s) \|\theta-\theta^*\|^2 \tau^2 \big( \underline{\gamma}-\sqrt{C_4 \mathbb{P} (A^c)} \big) 
	\end{split}
	\end{equation}
	where the first inequality holds from the equation \eqref{equ.L_s} and the intermediate value theorem on $\zeta (\theta^TX)-\zeta (\theta^{*T} X) $, the second inequality uses the assumption $\mathbb{E} XX^T \succeq \underline{\gamma} \tau^2$ in Assumption \ref{Assumption:binary} and Cauchy-Schwartz inequality. The third inequality holds from equation \eqref{equ.subgaussian1}. 
	
	Now we provide a bound on $\mathbb{P} (A^c)$.
	
	Using the the fact that $X$ is a $\tau^2$ sub-Gaussian, we have
	
	$$\mathbb{P} (A^c)=\mathbb{P} (\|UX\|_2\geq \frac{2s}{3r})\leq P( |U_1,X|\geq \frac{\sqrt{2}s}{3r})+P( |U_2,X|\geq \frac{\sqrt{2}s}{3r})\leq 4\exp (\frac{-s^2}{9r^2\tau^2}).$$
	Thus we have
	$$ \langle \theta-\theta^*, \nabla R(\theta)\rangle \geq 2L^2 (s) \|\theta-\theta^*\|^2_2 \tau^2 \big( \underline{\gamma}-2\sqrt{C_4}\exp(\frac{-s^2}{18r^2\tau^2})  \big) .$$
	
	We set $s=3\sqrt{2}C r\tau $ such that $ 2\sqrt{C_4}\exp (-\frac{-s^2}{18r^2\tau^2})\leq \frac{1}{2} \underline{\gamma}. $ 
	
	We then denote $\mu_0=\underline{\gamma} L^2(s) \tau^2$, then have
	\begin{equation}\label{equ.smallest_direction}
	\langle \theta-\theta^*, \nabla R(\theta) \rangle\geq \mu_0 \|\theta-\theta^*\|_2^2,
	\end{equation}	
	which implies $ \|\nabla R(\theta)\|_2 \geq 
	\mu_0 \|\theta-\theta^*\|_2$.
	
	Then we choose $ \underline{L_0}=\epsilon_0 \mu_0$, and have  $ \inf_{\theta \in B^p (0,r)\backslash B^p(\theta^*,\epsilon_0)} \|\nabla R(\theta)\|_2\geq \underline{L}_0$.

	For $\theta\in B^p (0,r)\backslash B^p (\theta^*, \epsilon_0/2),$ we have
	
	\begin{equation}
	\begin{split}
	&\inf_{\theta\in B^p (0,r)\backslash B^p (\theta^*, \epsilon_0/2)} \frac{\langle \nabla R_n(\theta),\theta-\theta^* \rangle}{\|\theta-\theta^*\|}\\
	&\geq \inf_{\inf_{\theta\in B^p (0,r)\backslash B^p (\theta^*, \epsilon_0/2)} } \frac{\langle \nabla R(\theta),\theta-\theta^* \rangle}{\|\theta-\theta^*\|}-\sup_{\theta\in B^p (0,r)\backslash B^p (\theta^*, \epsilon_0/2) }\| \nabla R_n(\theta)-\nabla R(\theta)\|_2\\
	&\geq \inf_{\theta\in B^p (0,r)\backslash B^p (\theta^*, \epsilon_0/2)}  \mu_0 \|\theta-\theta^*\|_2-\frac{1}{4} \mu_0 \epsilon_0\geq \frac{1}{4}\mu_0 \epsilon_0
	\end{split}
	\end{equation}
	Then we have
	$$ \|\theta-\theta^*\|_2\|\nabla R_n(\theta)\|_2\geq \frac{\mu_0}{4} \|\theta-\theta^*\|_2,$$
	Thus $\|\nabla R(\theta)\|_2>0.$
\end{proof}

Now we are ready to prove Theorem \ref{Theorem:Binary}. Basically, it has two steps. In the first step, we prove that $\theta^k$ converges to $\theta^*$ when $\theta \in B^p (0,r)\backslash B^p(\theta^*,\epsilon_0)$. Remark that $\theta^*$ is the true parameter rather than the optimal solution $\hat{\theta}$ of the empirical loss function $R_n(\theta)$. Then we use the third result in Lemma \ref{Lemma.bound_gradient_binary} which implies $\|\hat{\theta}-\theta^*\|_2\leq \frac{\epsilon_0}{2}$ and convert the convergence result to $\hat{\theta}$. The second step is to prove the convergence to $\hat{\theta}$ in $B^p(\theta^*,\epsilon_0)$, it is relatively easy, since we will show $R_n(\theta)$ is strongly convex.

\begin{proof}[Proof of Theorem \ref{Theorem:Binary}]
	
	We start with the batch case.
	
	Using Lemma \ref{Lemma.bound_gradient_binary}, we already know there exists a $\epsilon_0>0$, and $ \underline{L}_0>0$ and $\mu_0>0$ for all $\theta\in B^{d}(0,r)$, such that
	$$ \langle \theta-\theta^*, \nabla R(\theta) \rangle \geq \mu_0 \|\theta-\theta^*\|_2^2$$
	and
	$$ \inf_{\theta\in B^d (0,r)\backslash B^d(\theta^*,\epsilon_0)} \|\nabla R(\theta)\|_2\geq \underline{L}_0.$$
	
	Notice all constant $\epsilon_0$, $\mu_0$, $\underline{L}_0$ are functions of $ (\sigma(\cdot),r,\tau^2, L_\sigma, \underline{\gamma})$ but do not depend on dimension $p$ and the distribution of $X$. We have
	$$ \mu_0 \|\theta-\theta^*\|_2^2\leq \langle \theta-\theta^*, \nabla R(\theta) \rangle\leq \|\theta-\theta^*\|_2\|\nabla R(\theta)\|_2. $$
	Thus
	\begin{equation}\label{equ.EB}
	\mu_0\|\theta-\theta^*\|_2\leq \|\nabla R(\theta)\|_2.
	\end{equation}

	Using the smoothness of $R(\theta)$ we have
	\begin{equation}\label{equ.pl}
	\begin{split}
	R(\theta)-R(\theta^*)\leq & \langle \nabla R(\theta^*), \theta-\theta^* \rangle+\frac{L}{2} \|\theta-\theta^*\|_2^2\\
	\leq &\frac{L}{2} \|\theta-\theta^*\|_2^2\leq \frac{L}{2\mu_0^2} \|\nabla R(\theta)\|_2^2,
	\end{split}
	\end{equation}
	where we use the fact that $\nabla R(\theta^*)=0$ and equation \eqref{equ.EB}.
	
	Define $\epsilon=\sup_{\theta\in B^p (0,r)} \|\nabla R(\theta)-\nabla R_n(\theta)\|_2 $ and use the smoothness of $R(\theta)$, and choose $\eta=\frac{1}{2L}$, we have
	\begin{equation}
	\begin{split}
	R(\theta^{k+1})\leq & R(\theta^k)+\langle \nabla R(\theta^k),\theta^{k+1}-\theta^k \rangle+\frac{L}{2} \|\theta^{k+1}-\theta^k\|_2^2\\
	=& R(\theta^k)-\langle \nabla R(\theta^k), \eta \nabla R_n(\theta^k) \rangle+\frac{L\eta^2}{2}\|\nabla R_n(\theta^k)\|_2^2\\
	=& R(\theta^k)-\frac{1}{2L} \|\nabla R(\theta^k)\|_2^2-\frac{1}{2L}\langle \nabla R(\theta^k), \nabla R_n (\theta^k)-\nabla R(\theta^k)  \rangle+\frac{1}{8L} \|\nabla R_n(\theta^k)\|_2^2\\
	\leq & R(\theta^k)-\frac{1}{2L} \|\nabla R(\theta^k)\|_2^2+\frac{1}{2L} \|\nabla R(\theta^k)\|_2 \|\nabla R_n (\theta^k)-\nabla R(\theta^k)\|_2\\
	&+\frac{1}{4L} (\|\nabla R(\theta^k)\|_2^2+\|\nabla R_n(\theta^k)-\nabla R(\theta^k)\|_2^2)\\
	\leq& R(\theta^k)-\frac{1}{4L} \|\nabla R(\theta^k)\|_2^2+\frac{1}{2L}\epsilon\|\nabla R(\theta^k)\|_2+\frac{1}{4L}\epsilon^2.
	\end{split}
	\end{equation}

	Recall the Theorem \ref{Theorem:population}, which says that $\epsilon\leq \sqrt{\frac{p\log n}{n}}$ with high probability.
	
	Now we choose $\epsilon\leq \frac{1}{4} \underline{L}_0=\frac{1}{4} \epsilon_0 \mu_0$, which can be satisfied by setting $\frac{n}{\log n} \geq \frac{Cp}{\underline{L}_0^2}$  from some positive constant C. Here $\epsilon_0$ and $\mu_0$ are constant appearing in Lemma \ref{Lemma.bound_gradient_binary}, and they do not depend on $p$ and $n$.
	
	Notice we have $\inf_{\theta\in B^d (0,r)\backslash B^d(\theta^*,\epsilon_0)} \|\nabla R(\theta)\|_2\geq \underline{L}_0 $ by Lemma \ref{Lemma.bound_gradient_binary}, thus we have
	
	\begin{equation}\label{equ.important_step}
	R(\theta^{k+1})\leq R(\theta^k)-\frac{1}{16L}\|\nabla R(\theta^k)\|_2^2.
	\end{equation}

	Now combine above equation with the \eqref{equ.pl}, i.e., $R(\theta^k)-R(\theta^*)\leq \frac{L}{\mu_0^2} \|\nabla R(\theta^k)\|^2$, we obtain
	
	$$ R(\theta^{k+1})-R(\theta^*)\leq (1-\frac{\mu_0^2}{16L^2}) (R(\theta^k)-R(\theta^*)) .$$
	
	Thus we obtain a linear convergence rate on optimality gap

	$$R(\theta^k)-R(\theta^*)\leq (1-\frac{\mu_0^2}{16L^2})^k (R (\theta^0)-R(\theta^*))$$
	
	until it comes into the region $ B^d(\theta^*,\epsilon_0).$

	Notice it easy to transfer the convergence result on $R(\theta^k)$ to $\theta^k$. Using \eqref{equ.important_step}, we have
	
	$$ \|\nabla R(\theta^k)\|_2^2\leq 16L \left(R(\theta^{k})-R(\theta^{k+1})\right)\leq 16 L \left(R(\theta^{k})-R(\theta^*)\right).$$
	
	Combining this with \eqref{equ.EB}, we obtain
	
	$$\|\theta^k-\theta^*\|_2^2\leq \frac{16L}{\mu_0^2}(1-\frac{\mu_0^2}{16L^2})^k \left(R(\theta^0)-R(\theta^*)\right) $$
	
	Suppose at step $K$ it comes into the region  $ B^d(\theta^*,\epsilon_0)$, since $R(\theta)$ is strongly convex with parameter $\kappa_0$, we have
	
	$$R(\theta^K)-R(\theta^*)\geq\kappa_0 \|\theta_K-\theta^*\|_2^2,$$
	which implies following equation by Lemma \ref{Lemma.bound_gradient_binary} $$ \|\theta_K-\theta^*\|_2^2\leq \frac{1}{\kappa_0} (1-\frac{\mu_0^2}{16L^2})^K \left(R(\theta^0)-R(\theta^*)\right). $$
	
	Then we convert above result to the convergence to $\hat{\theta}$.

	\begin{equation}
	\begin{split}
	\|\theta^k-\hat{\theta}\|^2_2 &\leq 2\|\theta^k-\theta^*\|^2_2+2\|\theta^*-\hat{\theta}\|^2_2\leq 2\|\theta^k-\theta^*\|^2_2+2 (\frac{\epsilon_0}{2})^2\\
	&\leq  3 \|\theta^k-\theta^*\|^2_2\leq \frac{48L}{\mu_0^2}(1-\frac{\mu_0^2}{16L^2})^k \left(R(\theta^0)-R(\theta^*)\right),
	\end{split}
	\end{equation}
	where the second inequality holds from the third result in Lemma \ref{Lemma.bound_gradient_binary}, i.e., $\hat{\theta}$ only locate in $ B^p (\theta^*, \epsilon_0/2)$.

	So far, we obtain the linear convergence rate to $\hat{\theta}$ with ratio $1-\frac{\mu_0^2}{16L^2}$ in the region $ B^d (0,r)\backslash B^d(\theta^*,\epsilon_0)$.

	Then we prove the convergence in the region $B^d(\theta^*,\epsilon_0)$.
	We apply Theorem \ref{Theorem:population} on the Hessian of the empirical loss. We denote $\epsilon_h=\|\nabla^2 R(\theta)-\nabla^2 R_n(\theta)\|_{op}$.
	
	\begin{equation}
	\begin{split}
	\lambda_{\min} (\nabla^2 R_n(\theta))&\geq \lambda_{\min} (\nabla^2 R(\theta))-\|\nabla^2 R(\theta)-\nabla^2 R_n(\theta)\|_{op}\\ 
	& \geq \kappa_0-\epsilon_h\geq \frac{1}{2}\kappa_0,
	\end{split}
	\end{equation}
	
	where the last inequality holds by choosing $\frac{n}{\log n} \geq \frac{4C_0p}{\kappa_0^2}.$
	
	Thus in the region $B^d(\theta^*,\epsilon_0)$, $\hat{R}_n(\theta)$ is strongly convex with parameter $\frac{1}{2}\kappa_0.$
	
	According to the standard result smooth and strongly convex function using gradient method algorithm \cite{nesterov2013introductory}, we have the rate 
	
	$$ \|\theta^{K+t+1}-\hat{\theta}\|_2^2\leq (1-\frac{\kappa_0}{L}) \|\theta^{K+t}-\hat{\theta}\|_2^2. $$
	
	Notice $\kappa_0=\frac{1}{2}L^2(C\gamma \tau) \tau\underline{\gamma}$ and $\mu_0=L^2(C\gamma \tau) \tau\underline{\gamma} 
	$ (recall the the result of $\kappa_0 $ and $\mu_0$ in equation \eqref{equ.smallest_eig} and \eqref{equ.smallest_direction}), generally speaking, in the late stage the algorithm converge faster ( $1-\frac{\kappa_0}{L}$ v.s. $1-\frac{\mu_0^2}{L^2}$) .

	Now we begin to prove convergence of SVRG in the Theorem \ref{Theorem:Binary}.
	
	Similar to the batch case, we divide our proof into two step. The first step is in the region $ \theta\in B^d (0,r)\backslash B^d(\theta^*,\epsilon_0)$, we have the linear convergence to $\theta^*$. Then we use the fact that $ \|\hat{\theta}-\theta^*\|_2\leq \frac{1}{2}\epsilon_0$, to convert the convergence result to $\hat{\theta}$. Then in the region $B^p(\theta^*,\epsilon_0)$, we have the linear convergence to $\hat{\theta}$.

	Using smoothness on $R(\theta)$, we have
	\begin{equation}
	\begin{split}
	R(\theta^{k+1}_{s})\leq &R(\theta^k_{s})+\langle \nabla R(\theta^k_{s}), \theta^{k+1}_{s}-\theta^k_{s} \rangle+\frac{L}{2} \|\theta^{k+1}_{s}-\theta^{k}_s\|_2^2\\
	=&R(\theta^k_{s})+\langle \nabla R(\theta_s^k), -\eta v_s^k \rangle+\frac{L\eta^2}{2}\|v_s^{k}\|_2^2.\\
	\end{split}
	\end{equation}
	Thus 
	\begin{equation}\label{equ.smooth_svrg}
	\begin{split}
	\mathbb{E}(R(\theta_s)^{k+1})\leq &\mathbb{E} R(\theta_s^k)-\eta \mathbb{E}\langle \nabla R(\theta_s^k),\nabla R_n(\theta_s^k) \rangle+\frac{L\eta^2}{2}\mathbb{E}\|v_s^{k}\|_2^2\\
	=& \mathbb{E} R(\theta_s^k)-\eta \mathbb{E} \|\nabla R(\theta_s^k)\|_2^2+ \eta \mathbb{E}\|\nabla R(\theta_s^k)-\nabla R_n(\theta_s^k)\|_2\|\nabla R(\theta_s^k) \|_2+\frac{L\eta^2}{2}\mathbb{E}\|v_s^k\|_2^2\\
	\leq & \mathbb{E} R(\theta_s^k)-\eta \mathbb{E} \|\nabla R(\theta_s^k)\|_2^2+\frac{L\eta^2}{2}\mathbb{E}\|v_s^k\|_2^2+  \eta \mathbb{E}\epsilon\|\nabla R(\theta_s^k) \|_2.\\
	\end{split}
	\end{equation}
	
	Now we define $\Psi_s^k=\mathbb{E} \left(R(\theta_s^k)+c_k\|\theta_s^k-\tilde{\theta}_{s-1}\|_2^2\right)$, $c_m=0$ and specify the relation of $c_k$ and $c_{k+1}$ later.
	
	Then we bound $\Psi_{s}^{k+1}$ as follows
	
	\begin{equation}\label{equ.potential}
	\begin{split}
	\Psi_{s}^{k+1}=&\mathbb{E} R(\theta_s^{k+1})+c_{k+1} \mathbb{E}\|\theta_s^{k+1}-\tilde{\theta}_{s-1}\|_2^2\\
	=&\mathbb{E} R(\theta_s^{k+1})+c_{k+1} \mathbb{E}\|\theta_s^{k+1}-\theta_s^{k}+\theta_s^{k}-\tilde{\theta}_{s-1}\|_2
	^2 \\
	\leq&\mathbb{E} R(\theta_s^{k+1})+ c_{k+1}\mathbb{E} \|\theta_s^{k+1}-\theta_s^k\|_2^2+2c_{k+1} \mathbb{E}\langle -\eta v_s^k, \theta_s^k-\tilde{\theta}_s \rangle+c_{k+1}\mathbb{E} \|\theta_s^k-\tilde{\theta}_{s-1}\|_2^2\\
	=&\mathbb{E} R(\theta_s^{k+1})+ c_{k+1}\eta^2\mathbb{E} \|v_s^k\|_2^2-2c_{k+1}\eta\mathbb{E}\langle \nabla R_n(\theta_s^k), \theta_s^k-\tilde{\theta}_s \rangle+ c_{k+1}\mathbb{E} \|\theta_s^k-\tilde{\theta}_{s-1}\|_2^2\\
	\leq &\mathbb{E} R(\theta_s^{k+1})+ c_{k+1}\eta^2\mathbb{E}\|v_s^k\|_2^2+2c_{k+1}\eta\left( \frac{1}{2\alpha_k} \mathbb{E}\|\nabla R_n (\theta_s^k)\|_2^2+\frac{1}{2}\alpha_k \mathbb{E}\|\theta_s^k-\tilde{\theta}_{s-1}\|_2^2 \right)+c_{k+1}\mathbb{E} \|\theta_s^k-\tilde{\theta}_{s-1}\|_2^2.
	\end{split}
	\end{equation}
	
	Next we provide a bound on $\mathbb{E} \|v_s^k\|_2^2$ as follows. It is the Lemma 3 in \cite{reddi2016stochastic} and we preset here for completeness.
	
	\begin{equation}\label{equ.variance}
	\begin{split}
	\mathbb{E}\|v_s^k\|_2^2 =& \mathbb{E} \|\nabla \ell_{i_k}(\theta_s^{k})- \ell_{i_k} (\tilde{\theta}_{s-1})+\nabla R_n(\tilde{\theta}_{s-1}) \|_2^2\\
	=& \mathbb{E} \|\nabla \ell_{i_k}(\theta_s^{k})- \ell_{i_k} (\tilde{\theta}_{s-1})+\nabla R_n(\tilde{\theta}_{s-1})-\nabla R_n(\theta_s^k) +\nabla R_n(\theta_s^k)\|_2^2\\
	\leq & 2\mathbb{E}\|\nabla R_n(\theta_s^k)\|_2^2+2 \mathbb{E}\| \nabla \ell_{i_k}(\theta_s^{k})- \nabla\ell_{i_k} (\tilde{\theta}_{s-1})+\nabla R_n(\tilde{\theta}_{s-1})-\nabla R_n(\theta_s^k)\|_2^2\\
	\leq & 2\mathbb{E} \|\nabla R_n(\theta_s^k)\|_2^2+2\mathbb{E}\|\nabla \ell_{i_k}(\theta_s^{k})- \nabla\ell_{i_k} (\tilde{\theta}_{s-1}) \|_2^2\\
	\leq & 2\mathbb{E} \|\nabla R_n(\theta_s^k)\|_2^2+2L^2\mathbb{E}\|\theta_s^k-\tilde{\theta}_{s-1}\|_2^2.
	\end{split}
	\end{equation}
	
	Now we plug the result \eqref{equ.smooth_svrg}	and \eqref{equ.variance} in \eqref{equ.potential} and have
	
	\begin{equation}
	\begin{split}
	\mathbb{E} \Psi_s^{k+1} \leq &\mathbb{E} R(\theta_s^{k})- \eta\mathbb{E} \|\nabla R(\theta_s^k)\|_2^2+ \eta \mathbb{E} \epsilon \|\nabla R(\theta_s^k)\|_2+ \frac{c_{k+1}}{\alpha_k}\eta\mathbb{E} \|\nabla R_n(\theta_s^k)\|_2^2\\
	&(\frac{L\eta^2}{2}+c_{k+1}\eta^2) \mathbb{E} \|v_s^k\|_2^2+c_{k+1} (1+\eta\alpha_k) \mathbb{E} \|\theta_k-\tilde{\theta}_{s-1}\|_2^2\\
	\leq &  \mathbb{E} R(\theta_s^{k})- (\frac{3}{4}\eta-3\frac{c_{k+1}}{\alpha_k}\eta)\mathbb{E} \|\nabla R(\theta_s^k)\|_2^2+(\frac{L\eta^2}{2}+c_{k+1}\eta^2) \mathbb{E} \|v_s^k\|_2^2+c_{k+1} (1+\alpha_k\eta) \mathbb{E} \|\theta_k-\tilde{\theta}_{s-1}\|_2^2\\
	\leq & \mathbb{E} R(\theta_s^{k})+ \left( c_{k+1} (1+\alpha_k\eta+2L^2\eta^2) +L^3\eta^2\right) \mathbb{E} \|\theta_s^k-\tilde{\theta}_{s-1}\|_2^2- (\frac{3}{4}\eta-3\frac{c_{k+1}}{\alpha_k}\eta)\mathbb{E} \|\nabla R(\theta_s^k)\|_2^2\\
	&+(L\eta^2+2c_{k+1}\eta^2) \mathbb{E} \|\nabla R_n(\theta_s^k)\|_2^2\\
	\leq & \mathbb{E} R(\theta_s^{k})+ \left( c_{k+1} (1+\alpha_k\eta+2L^2\eta^2) +L^3\eta^2\right) \mathbb{E} \|\theta_s^k-\tilde{\theta}_{s-1}\|_2^2- (\frac{3}{4}\eta-3\frac{c_{k+1}}{\alpha_k}\eta)\mathbb{E} \|\nabla R(\theta_s^k)\|_2^2\\
	+&2 \left( L\eta^2+2c_{k+1} \eta^2\right) (\mathbb{E} ( \|R(\theta_s^k)\|_2^2+\epsilon^2))\\
	\leq &\mathbb{E} R(\theta_s^{k})+ \left( c_{k+1} (1+\alpha_k\eta+2L^2\eta^2) +L^3\eta^2\right) \mathbb{E} \|\theta_s^k-\tilde{\theta}_{s-1}\|_2^2\\
	& - (\frac{3}{4} \eta-3\frac{c_{k+1}}{\alpha_k}\eta-3L\eta^2-6c_{k+1}\eta^2) \mathbb{E} \|\nabla R(\theta_s^k)\|_2^2, 
	\end{split}
	\end{equation}
	where we use $\|\nabla R_n(\theta)\|_2^2\leq 2\| \nabla R(\theta)\|_2^2+2\epsilon^2$ and choose  $\epsilon\leq \frac{1}{4} \underline{L}_0=\frac{1}{4} \epsilon_0 \mu_0$ in the second  and last inequality , and \eqref{equ.variance} in the third inequality. Notice such requirement of $\epsilon$ can be satisfied by setting $\frac{n}{\log n} \geq  \frac{Cp}{\underline{L}_0^2}$ and the fact that in the region $\theta \in B^p (0,r)\backslash B^p(\theta^*,\epsilon_0)$ $
	\|\nabla R(\theta)\|_2\geq \underline{L}_0$ by Lemma \ref{Lemma.bound_gradient_binary}. The fourth inequality use $(a+b)^2\leq 2a^2+2b^2$.

	We assume that  $(\frac{3}{4}\eta-3\eta\frac{c_{k+1}}{\alpha_k}-3L\eta^2-6c_{k+1}\eta^2):=\beta_k>0 $ (we will verify it later by choosing $\eta$ and $\alpha_k$). 
	
	Recall our setting of $c_k=c_{k+1} (1+\alpha_k\eta+2L^2\eta^2) +L^3\eta^2 $, thus we have

	\begin{equation}
	\begin{split}
	\Psi_s^{k+1}\leq &\mathbb{E} R(\theta_s^k)+c_k \mathbb{E}\|\theta_s^k-\tilde{\theta}_{s-1}\|_2^2- (\frac{3}{4}\eta-3\eta\frac{c_{k+1}}{\alpha_k}-3L\eta^2-6c_{k+1}\eta^2)\\
	=& \Psi_s^{k}- \beta_k \mathbb{E} \|\nabla R(\theta_s^k)\|_2^2 .
	\end{split}
	\end{equation}

	We set $\beta:=\min_k{\beta_k}$, then we have 
	
	$$\mathbb{E} \|\nabla R(\theta_s^k)\|^2\leq \frac{\Psi_s^{k}-\Psi_s^{k+1}}{\beta}.  $$
	
	Then we sum over both side and have
	
	$$\sum_{k=0}^{m}\mathbb{E}\|\nabla R(\theta_s^k)\|^2 \leq \frac{\Psi_s^0-\Psi_s^m}{\beta}. $$

	Define $\Psi_s^0:=\mathbb{E} \left( R(\theta_s^0)+c_0 \|\theta_s^0-\tilde{\theta}_{s-1}\|_2^2 \right) $ and $\theta_s^0=\tilde{\theta}_{s-1} $  thus $\Psi_s^0= \mathbb{E} R(\theta_s^0)$.
	
	Remind that $c_{m}=0$ , $\Psi_s^m=\mathbb{E} R(\theta_s^m)=\mathbb{E} R(\theta_{s+1}^0)$, thus we have
	\begin{equation}\label{equ.svrg_tele}
	\frac{1}{T} \sum_{s=0}^{S-1}\sum_{k=0}^{m-1}\|\nabla R(\theta_s^k)\|_2^2 \leq \frac{R(\theta_0^0)-R(\theta_S^0)}{T\beta}\leq \frac{R(\theta_0^0)-R(\theta^*)}{T\beta}
	\end{equation}

	Next we bound $c_0$ and $ \beta$, recall that
	
	$$ c_k=c_{k+1} (1+\alpha_k\eta+2L^2\eta^2)+L^3\eta^2. $$
	
	We set $\alpha_k$ as a constant $\alpha$ and $\chi=\alpha\eta+2L^2\eta^2$, thus $ c_k=c_{k+1}(1+\chi)+L^3\eta^2.$
	We set $\eta=\kappa/(Ln^{2/3})$ for some positive constant $\kappa$, $\alpha=L/(n^{1/3})$ and $m=\lfloor n/2\kappa \rfloor $, thus $\chi=\frac{\kappa}{n}+\frac{2\kappa^2}{n^{4/3}}\leq 2\kappa/n.$
	
	$$c_0=\frac{\kappa L \left((1+\chi)^m-1\right
		)}{2\kappa+n^{1/3}}\leq\frac{\kappa L (1+\frac{2\kappa}{n})^{\frac{n}{2\kappa}}-1 }{2\kappa+n^{1/3}}\leq \kappa L (e-1) n^{-\frac{1}{3}}.$$
	
	Now we can lower bound $\beta$
	
	$$ \beta\geq \frac{3}{4}\eta-3\eta\frac{c_{0}}{\alpha}-3L\eta^2-6c_{0}\eta^2\geq\eta (\frac{3}{4}-3\kappa (e-1)-3\kappa n^{-\frac{1}{3}}-6\kappa^2 (e-1)/n).$$ Suppose $n$ is large enough and then we can choose $\kappa=0.4$ such that $\beta\geq\eta/2$. 
	
	Now we choose $T=\lceil 20L^2n^{2/3}/\mu_0^2 \rceil$, use the definition of $\theta^j$ and \eqref{equ.svrg_tele} then have
	
	$$ \mathbb{E} \|\nabla R(\theta^j)\|_2^2\leq \big( R(\theta^{j-1})-R(\theta^*)\big)/(4L/\mu_0^2) \leq \frac{1}{2} \mathbb{E}  \|\nabla R(\theta^{j-1})\|_2^2 .$$
	
	where the second equality holds from equation \eqref{equ.pl}.
	
	Thus
	\begin{equation}\label{equ.half_decrease}
	\mathbb{E} \|\nabla R(\theta^j)\|_2^2\leq  (\frac{1}{2})^j \mathbb{E}\|\nabla R(\theta^0)\|_2^2.
	\end{equation}

	Then use equation \eqref{equ.EB}, we obtain
	
	$$ \mathbb{E}\|\theta^j-\theta^*\|_2^2\leq \frac{1}{\mu_0^2} (\frac{1}{2})^j \mathbb{E}\|\nabla R(\theta^0)\|_2^2 \leq \frac{L^2}{\mu_0^2} (\frac{1}{2})^j \|\theta^0-\theta^*\|_2^2.$$
	
	Then we convert the convergence result to $\hat{\theta}$. Particularly,  we use the third result in Lemma \ref{Lemma.bound_gradient_binary}.
	
	\begin{equation}
	\begin{split}
	\|\theta^j-\hat{\theta}\|^2_2 &\leq 2\|\theta^j-\theta^*\|^2_2+2\|\theta^*-\hat{\theta}\|^2_2\leq 2\|\theta^j-\theta^*\|^2_2+2 (\frac{\epsilon_0}{2})^2\\
	&\leq  3 \|\theta^j-\theta^*\|^2_2\leq 3\frac{L^2}{\mu_0^2}(\frac{1}{2})^j \|\theta^0-\theta^*\|_2^2. 
	\end{split}
	\end{equation}
	
	Now we calculate the gradient complexity. In each iteration, we need $n+T=n+n^{2/3} \frac{L^2}{\mu_0^2} $ computations of gradient (we omit some contant here). Thus the result is $ \mathcal{O} ( \big(n+n^{2/3} \frac{L^2}{\mu_0^2} \big) \log(1/\varepsilon)) $.

	Then we prove the convergence in $B^p(\theta^*,\epsilon_0)$. Again we apply Theorem \ref{Theorem:population} on the Hessian of the empirical loss. We denote $\epsilon_h=\|\nabla^2 R(\theta)-\nabla^2 R_n(\theta)\|_{op}$.
	
	\begin{equation}
	\begin{split}
	\lambda_{\min} (\nabla^2 R_n(\theta))&\geq \lambda_{\min} (\nabla^2 R(\theta))-\|\nabla^2 R(\theta)-\nabla^2 R_n(\theta)\|_{op}\\ 
	& \geq \kappa_0-\epsilon_h\geq \frac{1}{2}\kappa_0,
	\end{split}
	\end{equation}
	
	where the last inequality holds by choosing $\frac{n}{\log n}\approx\frac{4C_0p}{\kappa_0^2}.$ Thus $R_n(\theta)$ is $\frac{1}{2}\kappa_0$ strongly convex. Notice this fact , the third result in Lemma  \ref{Lemma.bound_gradient_binary} and other mild condition implies there exist unique minimizer $\hat{\theta}$ by Theorem 4 in \cite{mei2016landscape}.
	
	Now we have
	
	$$ R_n(\theta')\geq R_n(\theta)+\langle \nabla R_n(\theta),\theta'-\theta \rangle+\frac{\kappa_0}{4}\|\theta'-\theta\|_2^2.$$
	Now minimize $\theta'$ at both side, we obtain
	$$ R_n(\theta)-R_n(\hat{\theta})\leq \frac{\|\nabla R_n(\theta)\|_2^2}{\kappa_0}.$$
	The following step is to similar to the case in $B^p(0,r)\backslash B^p(\theta^*,\epsilon_0)$ .
	Notice In the proof, the objective function is $R_n(\theta)$ rather than $R(\theta)$.
	We also need to replace $\frac{\mu_0^2}{L}$ by $\kappa_0$.

	Here we list different steps.
	
	Using the smoothness of $R_n(\theta)$, we have
	
	\begin{equation}
	R_n(\theta^{k+1}_{s})\leq R_n(\theta^k_{s})+\langle \nabla R_n(\theta_s^k), -\eta v_s^k \rangle+\frac{L\eta^2}{2}\|v_s^{k}\|_2^2.\\
	\end{equation}
	Thus 
	\begin{equation}\label{equ.empirical_smooth}
	\begin{split}
	\mathbb{E}(R_n(\theta_s)^{k+1})
	\leq  \mathbb{E} R_n(\theta_s^k)-\eta \mathbb{E} \|\nabla R_n(\theta_s^k)\|_2^2+\frac{L\eta^2}{2}\mathbb{E}\|v_s^k\|_2^2
	\end{split}
	\end{equation}
	Now we define $\Psi_s^k=\mathbb{E} \left(R_n(\theta_s^k)+c_k\|\theta_s^k-\tilde{\theta}_{s-1}\|_2^2\right)$. 
	Notice in the following derivation  we do not have the error term $\|\nabla R_n(\theta)-\nabla R(\theta)\|_2$, since we directly analyze $R_n(\theta)$. Now  we provide counterpart of   \eqref{equ.potential}, and notice the only difference is $R_n(\theta)$ rather than $R(\theta)$.

	\begin{equation}\label{equ.empirical_potential}
	\begin{split}
	\Psi_{s}^{k+1}=&\mathbb{E} R_n(\theta_s^{k+1})+c_{k+1} \mathbb{E}\|\theta_s^{k+1}-\tilde{\theta}_{s-1}\|_2^2\\
	\leq&\mathbb{E} R_n(\theta_s^{k+1})+ c_{k+1}\mathbb{E} \|\theta_s^{k+1}-\theta_s^k\|_2^2+2c_{k+1} \mathbb{E}\langle -\eta v_s^k, \theta_s^k-\tilde{\theta}_s \rangle+c_{k+1}\mathbb{E} \|\theta_s^k-\tilde{\theta}_{s-1}\|_2^2\\
	=&\mathbb{E} R_n(\theta_s^{k+1})+ c_{k+1}\eta^2\mathbb{E} \|v_s^k\|_2^2-2c_{k+1}\eta\mathbb{E}\langle \nabla R_n(\theta_s^k), \theta_s^k-\tilde{\theta}_s \rangle+ c_{k+1}\mathbb{E} \|\theta_s^k-\tilde{\theta}_{s-1}\|_2^2\\
	\leq &\mathbb{E} R_n(\theta_s^{k+1})+ c_{k+1}\eta^2\mathbb{E}\|v_s^k\|_2^2+2c_{k+1}\eta\left( \frac{1}{2\alpha_k} \mathbb{E}\|\nabla R_n (\theta_s^k)\|_2^2+\frac{1}{2}\alpha_k \mathbb{E}\|\theta_s^k-\tilde{\theta}_{s-1}\|_2^2 \right)+c_{k+1}\mathbb{E} \|\theta_s^k-\tilde{\theta}_{s-1}\|_2^2.
	\end{split}
	\end{equation}
	
	Now combine \eqref{equ.empirical_smooth} and \eqref{equ.empirical_potential} and \eqref{equ.variance} together, we have
	
	\begin{equation}
	\begin{split}
	\mathbb{E} \Psi_s^{k+1} 
	\leq \mathbb{E} R_n(\theta_s^{k})+ \left( c_{k+1} (1+\alpha_k\eta+2L^2\eta^2) +L^3\eta^2\right) \mathbb{E} \|\theta_s^k-\tilde{\theta}_{s-1}\|_2^2- ( \eta-\frac{c_{k+1}}{\alpha_k}\eta-3L\eta^2-6c_{k+1}\eta^2) \mathbb{E} \|\nabla R_n(\theta_s^k)\|_2^2. 
	\end{split}
	\end{equation} 
	
	Following step are same with that in the region $B^p(0,r)\backslash B^p(\theta^*,\epsilon_0)$ modulus some changes of constants. Thus the gradient complexity is 
	$$ \mathcal{O} \big( (n+n^{2/3} \frac{L}{\kappa_0}) \log(1/\varepsilon)\big).$$
	
	Notice $\kappa_0=\frac{1}{2}L^2(C\gamma \tau) \tau\underline{\gamma}$ and $\mu_0=L^2(C\gamma \tau) \tau\underline{\gamma} 
	$, thus $\kappa_0\geq \mu_0^2/L$ (suppose $L/\mu_0>2$, it is common in machine learning problem, especially when the problem is ill-conditioned). It means in $B^p(\theta^*,\epsilon_0)$ the algorithm converge faster than the region outside the neighborhood of $\theta^*$. 
	
	Thus the final gradient complexity is 
	
	$$ \mathcal{O} ( \big(n+n^{2/3} \frac{L^2}{\mu_0^2} \big) \log(1/\varepsilon)) .$$
	
\end{proof}

We start to proof the convergence of SAGA.
Before we start the main proof, we present the following lemma which is the Lemma 4 in \cite{reddi2016proximal}.
\begin{lemma}\label{Lemma:variance_saga}
	For the iterates $\theta^k$, $v^k$ and $\alpha_i^t$ where $t\in \{0,...,T-1  \}$ in SAGA, the following inequality holds:
	$\mathbb{E}\|R_n(\theta^k)-v^k\|_2^2\leq \frac{L^2}{nb} \sum_{i=1}^{n} \mathbb{E}\|\theta^k-\alpha_i^k\|_2^2.   $
	
\end{lemma}

\begin{proof}[Convergence of SAGA]
	We define $\epsilon=\sup_{\theta\in B^p (0,r)} \|\nabla R(\theta)-\nabla R_n(\theta)\|_2 $. Use the smoothness of $R(\theta)$, we obtain
	
	\begin{equation}
	\begin{split}
	\mathbb{E}R(\theta^{k+1})\leq & \mathbb{E}[R(\theta^k)+\langle \nabla R(\theta^k), \theta^{k+1}-\theta^k \rangle+\frac{L}{2} \|\theta^{k+1}-\theta^{k}\|_2^2]\\
	=&\mathbb{E} R(\theta^k)+\mathbb{E}\langle \nabla R(\theta^k), -\eta v^k \rangle+\frac{L\eta^2}{2}\mathbb{E}\|v^{k}\|_2^2.\\
	\leq & \mathbb{E} R(\theta^k)-\eta\mathbb{E} \|\nabla R(\theta^k)\|_2^2+(\frac{L\eta^2}{2}-\frac{\eta}{4})\mathbb{E} \|v^k\|_2^2+\eta \mathbb{E}\epsilon \|\nabla R(\theta^k)\|_2+\frac{\eta}{4}\mathbb{E}\|v^k\|_2^2,\\
	\end{split}
	\end{equation}	
	where we use the fact $\|\nabla R(\theta^k)-\nabla R_n(\theta^k)\|_2\leq \epsilon$.
	
	Notice $v^k$ is an unbiased estimator of $\nabla R_n(\theta_k).$ Thus $\mathbb{E} \|v^k\|^2=\mathbb{E} \|v^k-\nabla R_n(\theta^k)\|_2^2+\mathbb{E}\|\nabla R_n(\theta^k)\|_2^2.$
	
	Thus we have
	
	\begin{equation}\label{equ.saga_smooth}
	\begin{split}
	\mathbb{E} R(\theta^{k+1})\leq &\mathbb{E} R(\theta^k)-(\eta \mathbb{E} \|\nabla R(\theta^k)\|_2^2-\frac{\eta}{4} \mathbb{E} \|\nabla R_n(\theta^k)\|_2^2 )+\frac{\eta}{4}\mathbb{E}\|v^k-\nabla R_n(\theta^k)\|_2^2+\epsilon\eta \mathbb{E}\|\nabla R(\theta^k)\|_2+(\frac{L\eta^2}{2}-\frac{\eta}{4})\mathbb{E} \|v^k\|_2^2 \\
	\leq & \mathbb{E} R(\theta^k)-(\eta \mathbb{E}\|\nabla R(\theta^k)\|_2^2-\frac{\eta}{2} \mathbb{E} \|\nabla R(\theta^k)\|_2^2-\mathbb{E}\frac{\eta}{2} \epsilon^2  )+\epsilon\eta \mathbb{E} \|\nabla R(\theta^k)\|_2+\frac{\eta}{4}\mathbb{E}\|v^k-\nabla R_n(\theta^k)\|_2^2\\
	&+(\frac{L\eta^2}{2}-\frac{\eta}{4})\mathbb{E} \|v^k\|_2^2\\
	\leq & \mathbb{E} R(\theta^k)-(\frac{\eta}{2} \mathbb{E}\|\nabla R(\theta^k)\|_2^2-\mathbb{E}\frac{\eta}{2}\epsilon^2  )+\eta  \mathbb{E} \epsilon\|\nabla R(\theta^k)\|_2
	+\frac{\eta}{4}\frac{L^2}{nb}\sum_{i=1}^{n} \mathbb{E}\|\theta^k-\alpha_i^k\|_2^2\\
	&+(\frac{L\eta^2}{2}-\frac{\eta}{4})\mathbb{E} \|v^k\|_2^2,
	\end{split}
	\end{equation}
	where the second inequality uses $(a+b)^2\leq 2a^2+2b^2$, and the third inequality applies Lemma \ref{Lemma:variance_saga}.

	We define the Lyapunov function $\Psi_k=\mathbb{E} R(\theta^k)+\frac{c_k}{n}\sum_{i=1}^{n} \|\theta^k-\alpha_i^k\|_2^2$ with $c_k= c_{k+1}(1+\beta)(1-p)+\frac{\eta L^2}{4b} $ and $c_K=0$.
	
	\begin{equation}
	\begin{split}
	\Psi_{k+1}=&\mathbb{E} R(\theta^{k+1})+\frac{c_{k+1}}{n}\sum_{i=1}^{n} \mathbb{E}\|\theta^{k+1}-\alpha_i^{k+1}\|_2^2\\
	=&\mathbb{E} \{ R(\theta^{k+1})+\frac{c_{k+1}p}{n} \sum_{i=1}^{n}\|\theta^{k+1}-\theta^k\|_2^2+\frac{c_{k+1}(1-p)}{n} \sum_{i=1}^{n} \|\theta^{k+1}-\alpha_i^k\|_2^2\}\\
	\leq& \mathbb{E} \{ R(\theta^{k+1}) +\frac{c_{k+1}p}{n} \sum_{i=1}^{n}\|\theta^{k+1}-\theta^k\|_2^2+\frac{c_{k+1}(1-p)}{n} \sum_{i=1}^{n} (\|\theta^{k+1}-\theta^k\|_2^2+\|\theta^k-\alpha_i^k\|_2^2+2\langle \theta^{k+1}-\theta^k,\theta^k-\alpha_i^k \rangle ) \}\\
	\leq & \mathbb{E}\{  R(\theta^{k+1})+c_{k+1} (1+\frac{1-p}{\beta})\|\eta v^k \|_2^2+\frac{c_{k+1}(1+\beta)(1-p)}{n}\sum_{i=1}^{n}\|\theta^{k}-\alpha_i^k\|_2^2 \},
	\end{split}
	\end{equation}
	where the second equality uses the definition of $\alpha_k$ in the algorithm, $p=1-(1-\frac{1}{n})^b$. The first inequality uses the fact $2ab\leq \frac{a^2}{\beta}+\beta b^2.$
	
	Now we plug in the equation \eqref{equ.saga_smooth}, and have

	\begin{equation}
	\begin{split}
	\Psi_{k+1}\leq &\mathbb{E}\{ \big( \frac{c_{k+1}(1+\beta)(1-p)}{n}+\frac{\eta L^2}{4nb}\big)\sum_{i=1}^{n}\|\theta^k-\alpha_i^k\|_2^2+  \big(c_{k+1}(1+\frac{1-p}{\beta})\eta^2+\frac{L\eta^2}{2}-\frac{\eta}{4}\big) \|v^k\|_2^2\\
	&+   R(\theta^k)-(\frac{\eta}{2} \|\nabla R(\theta^k)\|_2^2-\frac{\eta}{2}\epsilon^2  )+\epsilon\eta  \|\nabla R(\theta^k)\|_2 \}.
	\end{split}
	\end{equation}
	
	Recall we set $c_k= c_{k+1}(1+\beta)(1-p)+\frac{\eta L^2}{4b} $. We claim $c_{k+1}(1+\frac{1-p}{\beta})\eta^2+\frac{L\eta^2}{2}-\frac{\eta}{4}\leq 0$  by following  the same argument in \cite{reddi2016proximal} (equation (29)) to bound $c_k$.
	
	We set  $b=(n^{2/3})$,$\eta=2\rho/L,\beta=b/4n$ and $\rho\leq\frac{1}{10}$ and $K=\lceil 3L/ \mu_0^2 \eta \rceil.$ Notice $c_K=0$ and $p=1-(1-\frac{1}{n})^b\geq 1-\frac{1}{1+b/n} \geq \frac{b}{2n}$ thus we have $ (1+\beta)(1-p)\leq 1-p+\beta\leq b/4n.$ Then we obtain	
	$$ c_k\leq c_{k+1} (1-\lambda)+\frac{\eta L^2}{4b},$$
	where $\lambda=b/4n$.
	
	Then recurse above equation, we have
	
	$$c_k\leq \frac{\eta L^2}{4b} (\frac{1-(1-\lambda)^{K-k}}{\lambda})\leq \frac{2n\rho L}{b^2}.$$
	
	Combine this upper bound of $c_k$ and the setting of $\eta$, $\beta$, $p$, $b$, we can easily verify
	$c_{k+1}(1+\frac{1-p}{\beta})\eta^2+\frac{L\eta^2}{2}-\frac{\eta}{4}\leq 0.$

	Thus we have 
	\begin{equation}\label{equ.potential_saga}
	\Psi_{k+1}\leq \Psi_k-(\frac{\eta}{2} \mathbb{E}\|\nabla R(\theta^k)\|_2^2-\frac{\eta}{2}\epsilon^2  )+\mathbb{E}\epsilon\eta  \|\nabla R(\theta^k)\|_2.
	\end{equation}

	We choose $\epsilon\leq \frac{1}{4}\underline{L}_0=\frac{1}{4} \epsilon_0 \mu_0$, which can be satisfy by setting $\frac{n}{\log n}\geq  \frac{16p}{\underline{L}_0^2}$ using the Theorem \ref{Theorem:population}. And notice by Lemma \ref{Lemma.bound_gradient_binary} we have for $\theta \in B^p (0,r)\backslash B^p(\theta^*,\epsilon_0)$ $
	\|\nabla R(\theta)\|_2\geq \underline{L}_0$.  Thus we have
	
	\begin{equation}\label{equ.saga_telescope}
	\Psi_{k+1}\leq \Psi_k-\frac{\eta}{3}\mathbb{E}\|\nabla R(\theta^k)\|_2^2.
	\end{equation}
	Summation over both side  we obtain,
	
	$$ \Psi_{K}\leq \Psi_0 -\frac{\eta}{3}\sum_{k=0}^{K-1} \mathbb{E}\|\nabla R(\theta^k)\|_2^2.$$
	Notice $\Psi_K=\mathbb{E}R(\theta^K)$ and $\Psi_0=\mathbb{E} R(\theta^0), $ then we get
	
	$$(\frac{\eta}{3}\sum_{k=0}^{K-1} \mathbb{E}\|\nabla R(\theta^k)\|_2^2)/K\leq \big(\mathbb{E} R(\theta^0)-\mathbb{E} R(\theta^K)\big)/K\leq \big(\mathbb{E} R(\theta^0)-\mathbb{E} R(\theta^*)\big)/K . $$ 
	
	Recall the definition of $\theta_j$,
	and in $j$th epoch notice $\theta^0=\theta_{j-1}$ by the algorithm, we have 
	
	$$ \frac{\eta}{3} \mathbb{E} \|\nabla R(\theta_j)\|_2^2\leq \big(\mathbb{E} R(\theta^0)-\mathbb{E} R(\theta^*)\big)/K\leq \frac{L}{2K\mu_0^2} \mathbb{E} \|\nabla R(\theta_{j-1})\|_2^2.$$
	
	Now we choose $K=\lceil\frac{3L}{\mu_0^2\eta}\rceil,$
	thus $$ \mathbb{E} \|\nabla R(\theta_j)\|_2^2\leq \frac{1}{2} \mathbb{E} \|\nabla R(\theta_{j-1})\|_2^2$$
	
	The following steps are same with that in SVRG (from equation \eqref{equ.half_decrease}). Now we calculate the gradient complexity. Choose $\eta=\frac{1}{5L}$ (recall $\eta=2\rho/L\leq \frac{1}{5L}$), the gradient complexity is $ \mathcal{O}(n+n^{2/3} \frac{L^2}{\mu_0^2}) \log_2(1/\varepsilon)  $ since $b\approx C_1n^{2/3}, K\approx C_2L^2/(\mu_0^2)$ with some absolute positive constant $C_1$ $C_2$.

	Then we prove the convergence in the region $B^p(\theta^*,\epsilon_0)$. Similar to the SVRG, we apply Theorem \ref{Theorem:population} on the Hessian of the empirical loss. We denote $\epsilon_h=\|\nabla^2 R(\theta)-\nabla^2 R_n(\theta)\|_{op}$.
	
	\begin{equation}
	\begin{split}
	\lambda_{\min} (\nabla^2 R_n(\theta))&\geq \lambda_{\min} (\nabla^2 R(\theta))-\|\nabla^2 R(\theta)-\nabla^2 R_n(\theta)\|_{op}\\ 
	& \geq \kappa_0-\epsilon_h\geq \frac{1}{2}\kappa_0,
	\end{split}
	\end{equation}
	
	where the last inequality holds by choosing $\frac{n}{\log n}\geq \frac{4C_0p}{\kappa_0^2}.$ Thus $R_n(\theta)$ is $\frac{1}{2}\kappa_0$ strongly convex and we have
	
	$$ R_n(\theta')\geq R_n(\theta)+\langle \nabla R_n(\theta),\theta'-\theta \rangle+\frac{\kappa_0}{4}\|\theta'-\theta\|_2^2.$$
	Now minimize $\theta'$ at both side, we obtain
	$$ R_n(\theta)-R_n(\hat{\theta})\leq \frac{\|\nabla R_n(\theta)\|_2^2}{\kappa_0}.$$

	The proof is similar to the case of SAGA in $B^p(0,r)\backslash B^p(\theta^*,\epsilon_0)$. Notice now we analyze $R_n(\theta)$ rather than $R(\theta)$ and we do not the error term $\epsilon:=\sup_{\theta\in B^p (0,r)} \|\nabla R(\theta)-\nabla R_n(\theta)\|_2$.
	
	Following similar analysis in equation \eqref{equ.saga_smooth}, we have
	
	\begin{equation}\
	\begin{split}
	\mathbb{E} R_n(\theta^{k+1})
	\leq  \mathbb{E} R_n(\theta^k)-\frac{\eta}{2} \mathbb{E}\|\nabla R(\theta^k)\|_2^2+
	+\frac{\eta}{2}\frac{L^2}{nb}\sum_{i=1}^{n} \mathbb{E}\|\theta^k-\alpha_i^k\|_2^2+(\frac{L\eta^2}{2}-\frac{\eta}{2})\mathbb{E} \|v^k\|_2^2
	\end{split}
	\end{equation}
	
	Then define $\Psi_k=\mathbb{E}R_n(\theta^k) +\frac{c_k}{n}\sum_{i=1}^{n} \|\theta^k-\alpha_i^k\|_2^2$.
	
	Following steps are same except  now we choose $K=\lceil \frac{6L}{\kappa_0 \eta} \rceil$ and some modulus changes constants.
	
	Thus in $B^p(\theta^*,\epsilon_0)$, we have the gradient complexity 
	
	$$ \mathcal{O} \big( (n+n^{2/3} \frac{L}{\kappa_0}) \log(1/\varepsilon)\big).  $$
	
	Again, 	notice $\kappa_0=\frac{1}{2}L^2(C\gamma \tau) \tau\underline{\gamma}$ and $\mu_0=L^2(C\gamma \tau) \tau\underline{\gamma} 
	$, thus $\kappa_0\geq \mu_0^2/L$ (suppose $L/\mu_0>2$, it is common in machine learning problem, especially when the problem is ill-conditioned). It means in $B^p(\theta^*,\epsilon_0)$ the algorithm converge faster than the region outside the neighborhood of $\theta^*$. 
	
	Thus the final gradient complexity is 
	
	$$ \mathcal{O} ( \big(n+n^{2/3} \frac{L^2}{\mu_0^2} \big) \log(1/\varepsilon)) $$.

\end{proof}

\section{Robust regression}
The following Lemma is similar to \cite{mei2016landscape}.
\begin{lemma}
	if Assumption \ref{Assumption:robust_regression} is satisfied and $\|\theta^*\|_2\leq \frac{r}{3}$, then we have 
	\begin{itemize}
		\item There exist $\epsilon_0$ and some positive constant $\kappa_0$ such that $$ \inf_{\theta\in B^p(\theta^*,\epsilon_0)} \lambda_{\min} (\nabla^2 R(\theta))\geq \kappa_0$$
		\item There exist some positive constant $\underline{L}_0$ and $\mu_0$ such that, 
		$$ \inf_{\theta\in B^p(\theta^*,\epsilon_0)} \|\nabla R(\theta)\|_2\geq \underline{L}_0$$,
		$$ \langle \theta-\theta^*,\nabla R(\theta)\rangle\geq  \mu_0\|\theta-\theta^*\|_2^2, \theta \in B^p(0,r) .$$ 
		\item  For all $\theta\in B^p (0,r)\backslash B^p (\theta^*, \epsilon_0/2)$, $\|\nabla R_n(\theta)\|_2>0$.
	\end{itemize}
\end{lemma}

\begin{proof}
	\begin{equation}
	\begin{split}
	u^T\nabla^2R(\theta^*)u&=\mathbb{E} \psi'(\epsilon) \langle X,u \rangle^2=\mathbb{E} \psi'(\epsilon) \mathbb{E} \langle X,u\rangle^2\\
	&=g'(0) \mathbb{E} \langle X,u \rangle^2\geq g'(0) \mathbb{E} \langle X,u \rangle^2=c_1\underline{\gamma}\tau^2 
	\end{split}
	\end{equation}
	Then we bound the operator norm of $\nabla^2 R(\theta)-\nabla^2 R(\theta^*)$.
	\begin{equation}
	\begin{split}
	|u^T (\nabla^2 R(\theta)-\nabla^2(\theta^*))u|=&| \mathbb{E} \{ (\psi'( \langle X, \theta^*-\theta \rangle+\epsilon )-\psi'(\epsilon))\langle X, u\rangle   \}   |\\
	&\overset{(a)}{=}| \mathbb{E} \{ \psi''(\xi) \langle X,\theta^*-\theta \rangle \langle X, u \rangle^2   \}   |\\
	&\leq L_\psi \mathbb{E} \{ |\langle X,\theta
	^*-\theta \rangle| \langle X,u \rangle^2 \}\\ 
	&\overset{(b)}{\leq} \big(  \mathbb{E} \langle \theta-\theta^*\rangle^2 \mathbb{E} \langle X,u \rangle^4 \big)\\
	&\overset{(c)}{\leq} L_\psi(\|\theta-\theta^*\|_2^2 \tau^6 C_4)^{1/2}\\
	&=L_\psi \sqrt{C_4}\|\theta-\theta^*\|_2 \tau^3	
	\end{split}
	\end{equation}
	where (a) uses the intermediate value theorem, the (b) uses Cauchy-Schwarz inequality, (c) uses equation \eqref{equ.subgaussian1} again.
	
	Now we choose $\epsilon_0=\frac{c_1\underline{\gamma}}{2L_\psi \sqrt{C_4}\tau}$ then in the region $ B^p(\theta^*,\epsilon_0)$ we have $\lambda_{\min}(\nabla^2 R(\theta))\geq \kappa_0:= \frac{c_1}{2}\underline{\gamma}\tau^2$.
	
	Then we bound the directional gradient. Again we define let $U\in \mathbb{R}^{2\times d}$ be an orthogonal transform from $\mathbb{R}^p$ to $\mathbb{R}^2$, whose row space contain $\theta$ and $\theta^*$ and $A=\{ \|UX\|_2\leq \frac{2s}{3r}  \}$. 
	\begin{equation}
	\begin{split}
	\langle \theta-\theta^* , \nabla R(\theta)\rangle=&\mathbb{E}\{ \mathbb{E} \psi(\langle \theta^*-\theta,X \rangle+\epsilon)\langle \theta^*-\theta,X \rangle\} \\
	=&\mathbb{E}\{ g(\langle\theta^*-\theta \rangle)\langle \theta^*-\theta,X \rangle  \}\\
	\geq& L(s) \mathbb{E}\{ \langle \theta-\theta^*,X \rangle^2-\langle \theta-\theta^*,X  \rangle^2 \mathbf{1}_{A^C}  \}\\
	\geq &  L(s) \{ \underline{\gamma}\tau^2 \|\theta-\theta^*\|_2^2- \big( \mathbb{E} \langle\theta-\theta^* \rangle^4 \mathbb{P}(A^c) \big)^{1/2}  \}\\
	\geq & L(s) \|\theta-\theta^*\|_2^2 \tau^2 (\underline{\gamma}-\sqrt{C_4\mathbb{P}(A^c)}.	
	\end{split}
	\end{equation}

	Similar to the proof of Lemma \ref{Lemma.bound_gradient_binary}, use the the fact that $X$ is a $\tau^2$ sub-Gaussian, we have
	
	$$\mathbb{P} (A^c)=\mathbb{P} (\|UX\|_2\geq \frac{2s}{3r})\leq P( |U_1,X|\geq \frac{\sqrt{2}s}{3r})+P( |U_2,X|\geq \frac{\sqrt{2}s}{3r})\leq 4\exp (\frac{-s^2}{9r^2\tau^2}).$$
	Thus we have
	$$ \langle \theta-\theta^*, \nabla R(\theta)\rangle \geq 2L^2 (s) \|\theta-\theta^*\|^2_2 \tau^2 \big( \underline{\gamma}-2\sqrt{C_4}\exp(\frac{-s^2}{18r^2\tau^2})  \big) .$$	
	
	We set $s=3\sqrt{2}C r\tau $ and $ 2\sqrt{C_4}\exp (-\frac{-s^2}{18r^2\tau^2})\leq \frac{1}{2} \underline{\gamma}. $ 
	
	We then denote $\mu_0=\underline{\gamma} L^2(s) \tau^2$, then have
	\begin{equation}\label{equ.smallest_direction_robust}
	\langle \theta-\theta^*, \nabla R(\theta) \rangle\geq \mu_0 \|\theta-\theta^*\|_2^2,
	\end{equation}	
	which implies $ \|\nabla R(\theta)\|_2 \geq 
	\mu_0 \|\theta-\theta^*\|$.

	For $\theta\in B^p (0,r)\backslash B^p (\theta^*, \epsilon_0/2),$ we have
	
	\begin{equation}
	\begin{split}
	&\inf_{\theta\in B^p (0,r)\backslash B^p (\theta^*, \epsilon_0/2)} \frac{\langle \nabla R_n(\theta),\theta-\theta^* \rangle}{\|\theta-\theta^*\|}\\
	&\geq \inf_{\inf_{\theta\in B^p (0,r)\backslash B^p (\theta^*, \epsilon_0/2)} } \frac{\langle \nabla R(\theta),\theta-\theta^* \rangle}{\|\theta-\theta^*\|}-\sup_{\theta\in B^p (0,r)\backslash B^p (\theta^*, \epsilon_0/2) }\| \nabla R_n(\theta)-\nabla R(\theta)\|_2\\
	&\geq \inf_{\theta\in B^p (0,r)\backslash B^p (\theta^*, \epsilon_0/2)}  \mu_0 \|\theta-\theta^*\|_2-\frac{1}{4} \mu_0 \epsilon_0\geq \frac{1}{4}\mu_0 \epsilon_0
	\end{split}
	\end{equation}
	Then we have
	$$ \|\theta-\theta^*\|_2\|\nabla R_n(\theta)\|_2\geq \frac{\mu_0}{4} \|\theta-\theta^*\|,$$
	Thus $\|\nabla R(\theta)\|_2>0.$
\end{proof}

\begin{proof}[Proof of Theorem \ref{Theorem:Robust regression}]
	Proof of the convergence of gradient decent method, SVRG and SAGA are basically same with that in binary classification case. What we do is to replace the corresponding $\underline{L_0} $ $\mu_0$, $\kappa_0$, follow same line and finish the proof. 
\end{proof}

\end{document}